\documentclass{article}
\pdfoutput=1

\usepackage[nonatbib,preprint]{neurips_2019}

\usepackage[numbers]{natbib}

 \usepackage{mathptmx}      

\usepackage[utf8]{inputenc} 
\usepackage[T1]{fontenc}    
\usepackage{hyperref}       
\usepackage{url}            
\usepackage{booktabs}       
\usepackage{amsfonts}       
\usepackage{nicefrac}       
\usepackage{microtype}      
\usepackage{amsmath}
\usepackage{amssymb}
\usepackage{amsthm}
\usepackage{graphicx}
\usepackage{bm}

\usepackage{caption}
\usepackage{subcaption}
\usepackage{multirow}
\usepackage{tabularx}

\usepackage{algorithm}
\usepackage{algorithmic}

\def\tr{^\top}

\providecommand{\norm}[1]{\lVert#1\rVert}

\DeclareMathOperator{\sgn}{sgn}

\theoremstyle{plain}
\newtheorem{proposition}{Proposition}

\begin{document}

\title{AdaSTE: An Adaptive Straight-Through Estimator to Train Binary Neural Networks}

\author{%
  Huu Le, Rasmus Kj\ae{}r H\o{}ier, Che-Tsung Lin and Christopher Zach \\
  Chalmers University of Technology\\
  Gothenburg, Sweden \\
  \texttt{huul,hier,chetsung,zach@chalmers.se}
}

\maketitle

\begin{abstract}
  We propose a new algorithm for training deep neural networks (DNNs) with
  binary weights. In particular, we first cast the problem of training binary
  neural networks (BiNNs) as a bilevel optimization instance and subsequently
  construct flexible relaxations of this bilevel program.  The resulting
  training method shares its algorithmic simplicity with several existing
  approaches to train BiNNs, in particular with the straight-through gradient
  estimator successfully employed in BinaryConnect and subsequent methods. In
  fact, our proposed method can be interpreted as an adaptive variant of the
  original straight-through estimator that conditionally (but not always) acts
  like a linear mapping in the backward pass of error propagation.
  Experimental results demonstrate that our new algorithm offers favorable
  performance compared to existing approaches.
\end{abstract}

\section{Introduction}

\label{sec:intro}
Deploying deep neural networks (DNNs) to computing hardware such as mobile and
IoT devices with limited computational and storage resources is becoming increasingly
relevant in practice, and hence training methods especially dedicated to quantized DNNs have
emerged as important research topics in recent years~\cite{Chen2019}.  In this work, we are
particularly interested in the special case of DNNs with binary weights
limited to $\{+1,-1\}$, since in this setting the computations at inference time largely
reduce to sole additions and subtractions.
Very abstractly, the task of learning in such binary weight neural networks
(BiNNs) can be formulated as an optimization program with binary constraints
on the network paramters, i.e.,
\begin{align}
  & \min\nolimits_w \ell(w) \qquad \text{s.t. } w \in \{-1, 1\}^d, \label{eq:binary_opt} \\
  &= \min\nolimits_{w \in \{-1,1\}^d} \mathbb{E}_{(x,y)\sim p_{\text{data}}}\left[ \psi(f(x,w), y) \right],
\end{align}
where $d$ is the dimensionality of the underlying parameters (i.e.\ all
network weights), $p_{\text{data}}$ is the training distribution and
$\psi$ is the training loss (such as the cross-entropy or squared Euclidean
error loss). $f(x;w)$ is the prediction of the DNN with weights $w$ for
input~$x$.

\begin{figure*}[htb]
  \centering
  \includegraphics[width=0.99\textwidth]{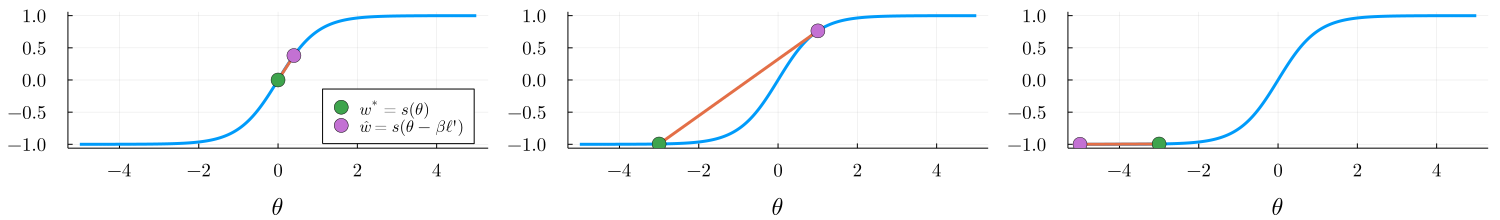}
  \caption{Adaptive straight-through estimation illustrated when $s$ is the
    $\tanh$ mapping. $\ell'$ is the incoming back-propagated error
    signal. Left: $\theta\approx 0$. The finite difference slope
    $(\hat w-w^*)/\beta$ matches the derivative of $\tanh$ very well. Middle:
    $\theta\ll 0$ and $\ell'<0$. A nearly vanishing derivative of
    $\tanh$ is boosted and $\tanh$ becomes ``leaky.''  Right:
    $\theta\ll 0$ and $\ell'>0$. No boosting of the gradient in this case. The
    case $\theta\gg 0$ is symmetrical.}
  \label{fig:teaser}
\end{figure*}

In practice, one needs to address problem settings where the parameter
dimension $d$ is very large (such as deep neural networks with many layers).
However, addressing the binary constraints in the above program is a
challenging task, which is due to the combinatorial and non-differentiable
nature of the underlying optimization problem.  In view of large training
datasets, (stochastic) gradient-based methods to obtain minimizers of
\eqref{eq:binary_opt} are highly preferable.
Various techniques have been proposed to address the above difficulties and
convert~\eqref{eq:binary_opt} into a differentiable surrogate.  The
general approach is to introduce real-valued ``latent'' weights
$\theta \in\mathbb{R}^d$, from which the effective weights
$w=\sgn(\theta)$ are generated via the sign function (or a differentiable
surrogate thereof).  One of the simplest and nevertheless highly successful
algorithms to train BiNNs termed BinaryConnect~\cite{Courbariaux2016} is based
on straight-through estimators (STE), which ignore the sign mapping entirely
when forming the gradient w.r.t.\ the latent weights $\theta$.  Although this
appears initially not justified, BinnaryConnect works surprisingly well and is
still a valid baseline method for comparison. More recently, the flexibility in
choosing the distance-like mapping leveraged in the mirror descent
method~\cite{nemirovskij1983problem} (and in particular the entropic descent
algorithm~\cite{beck2003mirror}) provides some justification of
BinaryConnect-like methods~\cite{Ajanthan2021} (see also Sec.~\ref{sec:mirror_descent}).

In this work, we propose a new framework for training binary neural networks.
In particular, we first formulate the training problem shown
in~\eqref{eq:binary_opt} as a bilevel optimization task, which is
subsequently relaxed using an optimal value reformulation. Further, we propose
a novel scheme to calculate meaningful gradient surrogates in order to update
the network parameters. The resulting method strongly resembles BinaryConnect
but leverages an adaptive variant of the straight-through gradient estimator:
the sign function is conditionally replaced by a suitable linear but
data-dependent mapping. Fig.~\ref{fig:teaser} illustrates the underlying
principle for the $\tanh$ mapping: depending on the incoming error signal,
vanishing gradients induced by $\tanh$ are conditionally replaced by
non-vanishing finite-difference surrogates. We finally point out that our
proposed method can be cast as a mirror descent method using a data-dependent
and varying distance-like mapping.

\section{Related Work}

The practical motivation for exploring weight quantization is to reduce  the
computational  costs  of  deploying  (and  in  some cases  training)  neural
networks. This can be particularly attractive in the case of edge computing and IoT devices~\cite{Chen2019}. Even when retaining floating point precision for activations $z$, using
binarized weights matrices $W$ means that the omnipresent product $Wz$ reduces to cheaper additions and subtractions of floating point values.

Already in the early 1990s,~\cite{Grossman1990,Venkatesh1993} trained BiNNs using fully local learning rules with layerwise targets computed via node perturbations.
In order to avoid the limited scalability of node perturbations,~\cite{Saad1990} instead employed a differentiable surrogate of the sign function for gradient computation.
Recently the use of differentiable surrogates in the backwards pass has been coined the Backward Pass Differentiable Approximation (BPDA) in the context of adversarial attacks~\cite{Athalye2018}. However, the same principle is at the core of many network quantization approaches, most notably the STE for gradient estimation.

Recent approaches have mainly focused on variations of the STE. A set of real valued (latent) weights are binarized when computing the forward pass, but during the backwards pass the identity mapping is used as its differentiable surrogate (which essentially makes the STE a special case of BPDA). The computed gradients are then used to update  the latent weights. The STE was presented by Hinton (and acredited to Krizhevsky) in a video lecture in 2012~\cite{Hinton2012}.
Subsequently it was employed for training networks with binary activations in~\cite{Bengio2013}, and to train networks with binary weights (and floating point activations) in the BinaryConnect (BC) model~\cite{Courbariaux2016}. BinaryConnect also used heuristics such as clipping the latent weights and employing Batch Normalization~\cite{Ioffe2015} (including its use at the output layer) to improve the performance of STE based training.
Further and recent analysis of the straight-through estimator is provided in~\cite{yin2018}, where its origin is traced back to early work on perceptrons Rosenblatt~\cite{Rosenblatt1957, Rosenblatt1962}.
The STE has also been applied to training fully binarized neural networks (e.g.~\cite{Hubara}). Moreover, Rastegari et al.~\cite{Rastegari2016} employ the STE for training fully binarized as well as mixed precision networks, and achieve improved performance by introducing layer and channel-wise scaling factors.
An interesting line of research has explored adapting the STE for variable bit-width quantization with learnable quantization step sizes in~\cite{esser2020} and learnable bit width in~\cite{uhlich2020}. 
\cite{uhlich2020} also introduces a regularization based method for constraining the learned bit-width to conform to a user-specified memory budget. 

Subsequent approaches have focused on deriving similar but less heuristic learning algorithms for networks with binary weights. ProxQuant (PQ)~\cite{Bai2019}, Proximal Mean-Field (PMF)~\cite{Ajanthan2019}, Mirror Descent (MD)~\cite{Ajanthan2021} and Rotated Binary Neural Networks (RBNN)~\cite{Lin2020} formulate the task of training DNNs with binary weights as a constrained optimization problem and propose different conversion functions used for moving between real-valued latent weights and binarized weights. A common feature among these methods is that they belong to the class of homotopy methods by gradually annealing the conversion mapping.
Qin et al~\cite{Qin2020a} introduce a novel technique for minimizing the information loss (caused by binarization) in the forward pass, and also aims to address gradient error by employing a gradually annealed tanh function as a differentiable surrogate during the backwards pass along with a carefully chosen gradient clipping schedule.
Similar to early research,~\cite{helwegen2019} does not introduce
latent real-valued weights, but rather updates the binary weights directly using a momentum based optimizer designed specifically for BiNNs. 
Several authors have approached the training of quantized neural networks via a variational approach~\cite{meng2020, soudry2014, Achterhold2018, louizos2018}.
Among those, BayesBiNN~\cite{meng2020} is particularly competitive: instead of optimizing over binary weights, the parameters of Bernoulli distributions are learned by employing both a Bayesian learning rule~\cite{khan2021} and the Gumbel-softmax trick~\cite{jang2016categorical,maddison2017concrete} (therefore requiring an inverse temperature parameter to convert the concrete distribution to a Bernoulli one).

For additional surveys of weight quantization we refer to the review papers~\cite{guo2018,Qin2020b} as well as section III of~\cite{Deng2020}. For a review of the efficacy of various ad-hoc techniques commonly employed for training BiNNs we refer to~\cite{alizadeh2018}.

\section{Background}

After clarifying some mathematical notations we summarize the mirror descent
method (and its use to train BiNNs) and the Prox-Quant approach
in order to better establish similarities and differences with our proposed
method later.

\subsection{Notation}

A constraint such as $w \in C$ is written as $\imath_C(w)$ in functional
form. We use $\odot$ to denote element-wise multiplication and
$\oslash$ for element-wise division.
The derivative of a function $\ell$ at $w$ is written as $\ell'(w)$. Many
mappings will be piece-wise differentiable but continuous. Therefore, in those cases
$\ell'(w)$ is a suitable element in the sub- or super-derivative.
We use an arrow over some variable names (especially $\vec\beta$) to emphasize
that this is a vector and not a scalar. For the same reason we use e.g.~$\vec s$ and $\vec\sgn$ to indicate the vectorized form of a scalar mapping $s$ (or $\sgn$) that is applied element-wise.

\subsection{Mirror Descent}
\label{sec:mirror_descent}

In short, mirror descent~\cite{nemirovskij1983problem,beck2003mirror}
successively generates new iterates by minimizing a regularized first-order
surrogate of the target objective. The most common quadratic regularizer
(which leads to the gradient descent method) is replaced by a more general
Bregmen divergence penalizing large deviations from the previous iterate. The
main motivation is to accelerate convergence of first-order methods, but it
can also yield very elegant methods such as the entropic descent algorithm,
where the utilized Bregman divergence based on the (negated) Shannon entropy
is identical to the KL divergence. The entropic descent method is very natural
when optimizing unknowns constrained to remain in the probability simplex
$\Delta$. The algorithm repeats updates of the form
\begin{align}
  w^{(t+1)} &\gets \arg\min_{w\in\Delta} w\tr\ell'(w^{(t)}) + \tfrac{1}{\eta} D_{KL}(w \| w^{(t)})
\end{align}
with the associated first-order optimality condition
\begin{align}
  w_j^{(t+1)} \propto w_j^{(t)} e^{-\eta \ell'(w^{(t)})_j}.
\end{align}
Reparametrizing $w$ as $w=\sigma(\theta)$, where $\sigma$ is the soft-arg-max
function, $\sigma(u)_j = e^{u_j}/\sum_{j'}e^{u_{j'}}$, yields
\begin{align}
  \theta^{(t+1)} &\gets \theta^{(t)} - \eta \ell'(w^{(t)}) = \theta^{(t)} - \eta\ell'(\sigma(\theta^{(t)})).
\end{align}
Interestingly, mirror descent modifies the chain rule by bypassing the inner
derivative, since the update is based on $\ell'(\sigma(\theta^{(t)}))$ and not
on $\frac{d}{d\theta} \ell(\sigma(\theta^{(t)}))$ as in regular gradient
descent. Hence, mirror descent is one way to justify the straight-through
estimator. The entropic descent algorithm is leveraged in~\cite{Ajanthan2021}
to train networks with binary (and also generally quantized) weights. The
soft-arg-max function $\sigma$ is slowly modified towards a hard arg-max
mapping in order to ultimately obtain strictly quantized weights.

\subsection{ProxQuant}

ProxQuant~\cite{Bai2019} is based on the observation that the straight-through
gradient estimator is linked to proximal operators via the dual averaging
method~\cite{xiao2010dual}. The proximal operator for a function $\phi$ is the
solution of the following least-squares regularized optimization problem,
\begin{align}
    \text{prox}_{\lambda\phi}(\theta) = \arg\min\nolimits_{\theta'} \lambda \phi(\theta') + \tfrac{1}{2} \|\theta' - \theta\|^2,
\end{align}
where $\lambda>0$ controls the regularization strength. If $\phi$ is a convex and
lower semi-continuous mapping, the minimizer of the r.h.s.\ is always unique
and $\text{prox}_{\lambda\phi}$ is a proper function (and plays an crucial role
in many convex optimization methods). ProxQuant uses a non-convex mappings for
$\phi$, which is far more uncommon for proximal steps than the convex case (see
e.g.~\cite{strekalovskiy2014real} for another example). In order to train DNNs
with binary weights, $\phi$ is chosen as W-shaped function,
\begin{align}
  \phi(\theta) = \sum\nolimits_{j=1}^d \min\left\{|\theta_j - 1|, |\theta_j + 1| \right\}.
\end{align}
$\phi$ has $2^d$ isolated global minima and is therefore not convex. Note that
$\text{prox}_{\lambda\phi}(\theta)$ is uniquely defined as long as all elements
in $\theta$ are non-zero. The network weights are updated according to
\begin{align}
  \theta^{(t+1)} \gets \text{prox}_{\lambda^{(t)} \phi} \left(\theta^{(t)} - \eta \ell'(\theta^{(t)}) \right),
\end{align}
and the regularization weight $\lambda^{(t)}$ is increased via an annealing
schedule, which makes ProxQuant an instance of homotopy methods: strictly
quantized weights are only obtained for a sufficiently large value of
$\lambda^{(t)}$.

\section{Adaptive Straight-Through Estimator}
In this section, we propose a new approach to tackle the optimization problem
given in~\eqref{eq:binary_opt}. Reformulating and relaxing an underlying
bilevel minimization problem is at the core of the proposed method.

\subsection{Bilevel Optimization Formulation}
\label{sec:bilevel}

We start by rewriting the original problem~\eqref{eq:binary_opt} as the
following bilevel minimization program,
\begin{align}
  \min\nolimits_{\theta, w} \ell(w^*) \quad \text{s.t. } w^* = \arg\min\nolimits_{w} \mathcal{E}(w;\theta)
  \label{eq:bilevel}
\end{align}
where $\mathcal E(w;\theta)$ can be any function that favors $w^*$ to be
binary. Two classical choices for $\mathcal E$ are given by
\begin{align}
  \mathcal E_{\text{tanh}}(w;\theta) &= -\tfrac{1}{\tau} \sum\nolimits_j H\left( \tfrac{1}{2}(1-w_j) \right) - w\tr\theta \\
  \mathcal E_{\text{hard-tanh}}(w;\theta) &= \tfrac{1}{2\tau} \norm{w}^2 - w\tr\theta + \imath_{[-1,1]^d}(w),
\end{align}
where $H$ is the Shannon entropy of a Bernoulli random variable,
$H(u) = u\log u+(1-u)\log(1-u)$. The minimizer $w^*$ for given
$\theta$ is the $\tanh$ mapping in the case of
$\mathcal E_{\text{tanh}}$, $w_j^*=\tanh(\theta_j/\tau)$, and the second
option yields the hard-tanh mapping,
$w_j^* = \Pi_{[-1,1]}(\theta_j/\tau)$. $\tau>0$ is a parameter steering how
well these mappings approximate the sign function $\vec\sgn(\theta)$.

In order to apply a gradient-based learning method we require that
$\mathcal E$ is differentiable w.r.t.\ $\theta$ for all $w$. In the above
examples we have
$\frac{\partial}{\partial\theta} \mathcal E(w;\theta)=-w$. It will be
sufficient for our purposes to assume that $\mathcal E$ is of the form
\begin{align}
  \mathcal E(w;\theta) = -w\tr \theta + \mathcal G(w)
  \label{eq:fenchel_E}
\end{align}
for a coercive function $\mathcal G$ bounded from below. That is, $w$ and
$\theta$ only interact via their (separable) inner product. Further, it is
sufficient to assume that $\mathcal G$ is fully separable,
$\mathcal G(w) = \sum_j G(w_j)$, since each latent weight $\theta_j$ can be
mapped to its binarized surrogate $w_j$ independently (an underlying assumption in the majority of works but explicitly deviated from in~\cite{han2020training}). Thus, the general form
for $\mathcal E$ assumed in the following is given by
\begin{align}
  \mathcal E(w;\theta) = \sum\nolimits_j \big( G(w_j) - w_j\theta_j \big).
  \label{eq:separable_E}
\end{align}
Therefore in this setting the solution $w^*=(w_1^*,\dotsc,w_d^*)\tr$ is given
element-wise,
\begin{align}
  w_j^* = \arg\min\nolimits_{w_j} G(w_j) - w_j\theta_j.
\end{align}

\subsection{Relaxing by Optimal Value Reformulation}

The optimal value reformulation (e.g.~\cite{outrata1988note, Zach2021}), which
is a commonly used reformulation approach in bilevel optimization, allows us
to rewrite the bilevel problem~\eqref{eq:bilevel} as follows,
\begin{align}
  \min\nolimits_{\theta, w} \ell(w) \quad \text{s.t. } \mathcal E(w;\theta) \le \min\nolimits_{w'} \mathcal E(w';\theta).
  \label{eq:bilevel_ovr}
\end{align}
Observe that the $w^*$ in the outer objective of~\eqref{eq:bilevel} was
replaced by a new unknown $w$, while the difficult equality constraint
in~\eqref{eq:bilevel} has been replaced by a somewhat easier inequality
constraint. Due to the separable nature of $\mathcal E$
in~\eqref{eq:separable_E}, it is advantageous to introduce an inequality
constraint for each element $w_j$. Thus, we obtain
\begin{align}
  \min\nolimits_{\theta, w} \ell(w) \quad \text{s.t. } E(w_j;\theta_j) \le \min\nolimits_{w_j'} E(w_j';\theta_j),
  \label{eq:separable_ovr}
\end{align}
where $E$ (independent of $j$) is given as
\begin{align}
  E(w_j;\theta_j) := G(w_j) - w_j\theta_j.
\end{align}
This first step enables us to straightforwardly
relax~\eqref{eq:separable_ovr} by fixing positive Lagrange multipliers for the
inequality constraints:
\begin{align}
  \min_{\theta, w} \ell(w) + \sum\nolimits_j \tfrac{1}{\beta_j} \big( E(w_j;\theta_j) - \min\nolimits_{w_j'} E(w_j';\theta_j) \big).
\end{align}
We parametrize the non-negative multipliers via $\beta_j^{-1}$ for
$\beta_j>0$, which will be convenient in the following. Since we are
interested in gradient-based methods, we replace the typically highly
non-convex ``loss'' $\ell$ (which subsumes the target loss and the mapping
induced by the network) by its linearization at $w^*$,
$\ell(w^*) + (w-w^*)\tr \ell'(w^*)$.
Recall that $w^* = \arg\min_w \mathcal E(w;\theta)$ is the effective weight
used in the DNN and is ideally close to $\vec\sgn(\theta)$. Overall, we arrive at
the following relaxed objective to train a network with binary weights:
\begin{align}
  \mathcal{L}(\theta) &= \ell(w^*) - (w^*)\tr \ell'(w^*)  \nonumber \\
  {} &+ \sum\nolimits_j \min_{w_j} \left\{ w_j \ell'_j(w^*) + \tfrac{1}{\beta_j} E(w_j; \theta_j) \right\} \nonumber \\
  {} &- \sum\nolimits_j \min_{w_j} \left\{ \tfrac{1}{\beta_j}E(w_j;\theta_j) \right\}
       \label{eq:contrastive_L},
\end{align}
The inner minimization problems have the solutions
\begin{align}
  w_j^* &= \arg\min\nolimits_{w_j} E(w_j;\theta_j) \qquad\text{and} \nonumber \\
  \hat w_j &:= \arg\min\nolimits_{w_j} \beta_j \ell'_j(w^*) w_j + E(w_j;\theta_j).
             \label{eq:inner}
\end{align}
$\hat w=(\hat w_1,\dotsc,\hat w_d)\tr$ is based on a perturbed objective that
incorporates the local (first-order) behavior of the outer loss
$\ell$. Both $w^*$ and $\hat w$ implicitly depend on the current value of
$\theta$, and $\hat w$ depends on a chosen ``step size'' vector
$\vec\beta:=(\beta_j)_{j=1}^d$ with each $\beta_j>0$. If
$\mathcal E(\cdot;\theta)$ is continuous at $w=w^*$, then
$\lim_{\beta_j\to 0^+} \hat w_j = w_j^*$. Further, if $\mathcal E$ is of the
form given in~\eqref{eq:fenchel_E}, then $\hat w$ is as easy to compute as
$w^*$:
\begin{proposition}
  Let $\mathcal E(w;\theta) = G(w) - w\tr\theta$ and
  $w^*=\arg\min_w \mathcal E(w;\theta)$ be explicitly given as
  $w^*=\vec s(\theta)$.  Then
  \begin{align}
    \hat w = \vec s\big( \theta - \vec\beta \odot \ell'(w^*) \big).
  \end{align}
\end{proposition}
\begin{proof}
  We simply absorb the linear perturbation term into $\theta$, yielding
  $\tilde\theta := \theta - \vec\beta\odot\ell'(w^*)$, and therefore
  $\hat w$ solves
  \begin{align}
    \hat w = \arg\min\nolimits_w G(w) -w\tr \tilde\theta = \arg\min\nolimits_w \mathcal E(w;\tilde\theta).
  \end{align}
  Hence, $\hat w = \vec s(\tilde\theta) = \vec s(\theta-\vec\beta\odot\ell'(w^*))$ as claimed.
\end{proof}
All of the interesting choices $\mathcal E$ lead to efficient forward mappings
$s$ (like the choices $\mathcal E_{\text{tanh}}$ and
$\mathcal E_{\text{hard-tanh}}$ given earlier that resulted in tanh and hard
tanh functions).

\subsection{Updating the latent weights $\theta$}
\label{sec:updating_theta}

For a fixed choice of $\vec\beta = (\beta_1,\dotsc,\beta_d)\tr$ with
$\beta_j>0$, the relaxed objective $\mathcal L(\theta)$
in~\eqref{eq:contrastive_L} is a nested minimization instance with a
``min-min-max'' structure. In some cases it is possible to obtain a pure
``min-min-min'' instance via duality~\cite{zach2019contrastive}, but in
practice this is not necessary. Let $\theta^{(t)}$ be the current solution at
iteration $t$, then our employed local model to determine the new iterate
$\theta^{(t+1)}$ is given by
\begin{align}
  Q(\theta;\theta^{(t)}) &= \sum\nolimits_j \tfrac{1}{\beta_j} \big( E(\hat w_j;\theta_j) - E(w_j^*;\theta_j) \big) \nonumber \\
  {} &+ \tfrac{1}{2\eta} \norm{\theta-\theta^{(t)}}^2,
\end{align}
where $w^*=\vec s(\theta^{(t)})$ and
$\hat w = \vec s(\theta^{(t)}-\vec\beta\odot\ell'(w^*))$
are the effective weights and its perturbed instance, respectively, evaluated
at $\theta^{(t)}$. The last term in $Q$ regularizes deviations from
$\theta^{(t)}$, and $\eta$ plays the role of the learning rate. Minimizing
$Q(\theta;\theta^{(t)})$ w.r.t.\ $\theta$ yields a gradient descent-like update,
\begin{align}
  \theta^{(t+1)} &= \arg\min\nolimits_\theta Q(\theta;\theta^{(t)}) \nonumber \\
  {} &= \theta^{(t)} - \eta \big( w^* - \hat w \big) \oslash \vec\beta
  \label{eq:theta_update}
\end{align}
for the assumed form of $\mathcal E$ in~\eqref{eq:fenchel_E}.
Each element of $(w^*-\hat w)\oslash \vec\beta$, i.e.\ 
$(w_j^*-\hat w_j)/\beta_j$, corresponds to a finite difference approximation
(using backward differences) of
\begin{align}
  -\tfrac{d}{d\beta_j} s\big( \theta^{(t)}_j - \beta_j\ell'_j(w^*) \big)\big|_{\beta_j=0^+}
\end{align}
with spacing parameter $h_j = \beta_j \ell'_j(w^*)$. If $s$ is at least
one-sided differentiable, then it can be shown that these
finite differences converge to a derivative given by the chain rule when
$\beta_j\to 0^+$~\cite{Zach2021},
\begin{align}
  \tfrac{1}{\beta_j} \big( w_j^* &- \hat w_j \big)
                       \stackrel{\beta_j\to 0^+} \to -\tfrac{d}{d\beta} s(\theta^{(t)}_j - \beta_j\ell'_j(w^*))\big|_{\beta_j=0^+} \nonumber \\
  {} &= \ell_j'(s(\theta^{(t)}_j)) s'(\theta^{(t)}_j) = \tfrac{d}{d\theta_j} \ell(s(\theta^{(t)})).
\end{align}
For non-infintesimal $\beta_j>0$ the finite difference slope
$(w_j^*-\hat w_j)/\beta_j$ corresponds to a perturbed chain rule,
\begin{align}
  \tfrac{1}{\beta_j} \big( w^*_j - \hat w_j \big) = \ell'_j(w^*) s'\big( \theta_j^{(t)}-\gamma_j \ell'_j(w^*) \big)
  \label{eq:mod_chain_rule}
\end{align}
(recall that $w^*=s(\theta^{(t)})$), where the inner derivative is evaluated at
a perturbed argument $\theta^{(t)}-\vec\gamma\odot\ell'(w^*)$ for a
$\vec\gamma \in [0,\vec\beta]$. This is a consequence of the mean value theorem.
Moreover, if each $\beta_j$ is a stationary point of the mapping
\begin{align}
    \beta\mapsto \tfrac{1}{\beta} \big( w_j^*-\hat w_j \big) = \tfrac{1}{\beta} \big( w_j^*-s(\theta_j^{(t)} - \beta\ell'_j(w^*)) \big),
\end{align}
then by using the quotient rule it is easy to see that $\vec\gamma=\vec\beta$, and therefore
\begin{align}
  \tfrac{1}{\beta_j} \big( w_j^* - \hat w_j \big) = \ell_j'(w^*) s'\big( \theta^{(t)}_j - \beta_j\ell'_j(w^*) \big).
\end{align}
Additionally, the relation in~\eqref{eq:mod_chain_rule} can be interpreted as a particular instance of mirror descent (recall Sec.~\ref{sec:mirror_descent}) as shown in the appendix.
Overall, the above means that we can relatively freely select where $s'$ is actually
evaluated. Since $s$ is naturally a ``squashing'' function mapping
$\mathbb R$ to the bounded interval $[-1,1]$, gradient-based training using
$s'$ usually suffers from the vanishing gradient problem.  Using the relaxed
reformulation for bi\-level programs allows us to select $\beta_j$ to obtain a
desired descent direction as it will be described in
Section~\ref{sec:adaptive_beta}.

The resulting gradient-based training method is summarized in
Alg.~\ref{alg:training}. The algorithm is stated as full batch method, but the
extension to stochastic variants working with mini-batches drawn from
$p_{\text{data}}$ is straightforward. In the following section we discuss our
choice of $\mathcal E$ and how to select suitable spacing parameters
$\vec\beta^{(t)}>0$ in each iteration. Since $\vec\beta^{(t)}$ is chosen
adaptively based on the values of $\theta^{(t)}$ and $\ell'(w^*)$ and used to
perturb the chain rule, we call the resulting algorithm the \emph{adaptive
  straight-through estimator} (AdaSTE) training method.

\begin{algorithm}[htb]
  \begin{algorithmic}[1]
    \STATE{Initialize $\theta^{(0)}$, choose learning rates $\eta^{(t)}$, $t=1,\dotsc$}
    \FOR{$t=1,\dotsc$}
    \STATE{$w^*\gets \vec s(\theta^{(t)})$}
    \STATE{Run regular back-propagation to determine $\ell'(w^*)$}
    \STATE{Determine $\vec\beta^{(t)}$ using~\eqref{eq:beta_choice}}
    \STATE{$\hat w \gets \vec{s}\big( \theta^{(t)} - \vec\beta^{(t)}\odot \ell'(w^*) \big)$}
    \STATE{$\theta^{(t+1)} \gets \theta^{(t)} - \eta^{(t)} (w^*-\hat w) \oslash \vec\beta^{(t)}$}
    \ENDFOR
  \end{algorithmic}
  \caption{AdaSTE training method.}
  \label{alg:training}
\end{algorithm}

\subsection{Our choice for the inner objective $\cal E$}
\label{sec:new_reg}

In this section we will specify our choice for $\mathcal{E}$ (and thus the
mapping $\vec s:\theta \mapsto \arg\min_w \mathcal E(w;\theta)$). The straightforward
options of $\mathcal E_{\text{tanh}}$ and
$\mathcal E_{\text{hard-tanh}}$ (Section~\ref{sec:bilevel}) suffer from the
fact that the induced arg-min mappings coincide exactly with the sign function
only when the hyper-parameter $\tau^{-1}\to\infty$. We are interested in an
inner objective $\mathcal E$ that yields perfect quanitized mappings for
finite-valued choices of hyper-parameters. Inspired by the double-well cost
used in ProxQuant~\cite{Bai2019}, we design $\mathcal E$ as follows,
\begin{align}
  \mathcal E(w;\theta) &= \tfrac{1+\mu}{2} \norm{w}^2 - w\tr\theta - \mu (1\!+\!\alpha) \norm{w}_1 + \imath_{[-1,1]^d}(w),
\end{align}
where $\mu>0$ and $\alpha\in(0,1)$ are free parameters. Note that
$\mathcal E$ is only piecewise convex in $w$ for fixed $\theta$, but it is
fully separable in $w_j$ with
\begin{align}
  E(w_j;\theta_j) = \tfrac{1+\mu}{2} w_j^2 - w_j\theta_j - \mu (1\!+\!\alpha) |w_j| + \imath_{[-1,1]}(w_j).
  \label{eq:our_E}
\end{align}
Via algebraic manipulations we find the following closed-form expression for
$\hat w_j$ (where we abbreviate $\ell'$ for $\ell'(w^*)$),
\begin{align}
    \hat w_j &= \arg\min_{w_j} \beta_j \ell'_j w_j + E(w_j; \theta_j) \nonumber \\
    &= \Pi_{[-1,1]} \left( \frac{\tilde\theta_j + \mu(1+\alpha) \sgn(\tilde\theta_j)}{1+\mu} \right),
      \label{eq:our_s_mapping}
\end{align}
with $\tilde\theta_j := \theta_j-\beta_j\ell_j'$. In other words, the forward
mapping $\vec s:\theta \mapsto w^*=s(\theta)$ for our choice of $\mathcal E$ is given by
\begin{align}
  \vec s(\theta) &= \Pi_{[-1,1]^d} \left( \frac{\theta + \mu(1+\alpha) \vec\sgn(\theta)}{1+\mu} \right).
\end{align}
The piece-wise linear graph of this mapping is illustrated in
Fig.~\ref{fig:activation} for $\alpha=1/100$ and three different choices of $\mu$. Let
$\alpha\in(0,1)$ be given, then $\vec s(\theta)$ attains only values in
$\{-1,1\}^d$ even for finite $\mu$, since
\begin{align}
    &\frac{|\theta_j| + \mu(1+\alpha)}{1+\mu} \ge 1 \iff |\theta_j| + \mu(1+\alpha) \ge 1 + \mu \nonumber \\
    &\iff  |\theta_j| + \alpha\mu \ge 1,
\end{align}
which implies that any $\theta_j$ is always mapped to +1 or -1 when
$\mu\ge 1/\alpha$ (and the exact values of $\mu$ and $\alpha$ do not matter in this case).
Consequently we have both the option to train with strictly
binary weights from the beginning, or to train via a homotopy method by
adjusting $\alpha$ or $\mu$. Both choices lead to competitive results with the
homotopy-based method having a small advantage in some cases as demonstrated
in Section~\ref{sec:results}.

\begin{figure}[tb]
  \centering
  \includegraphics[width=0.65\textwidth]{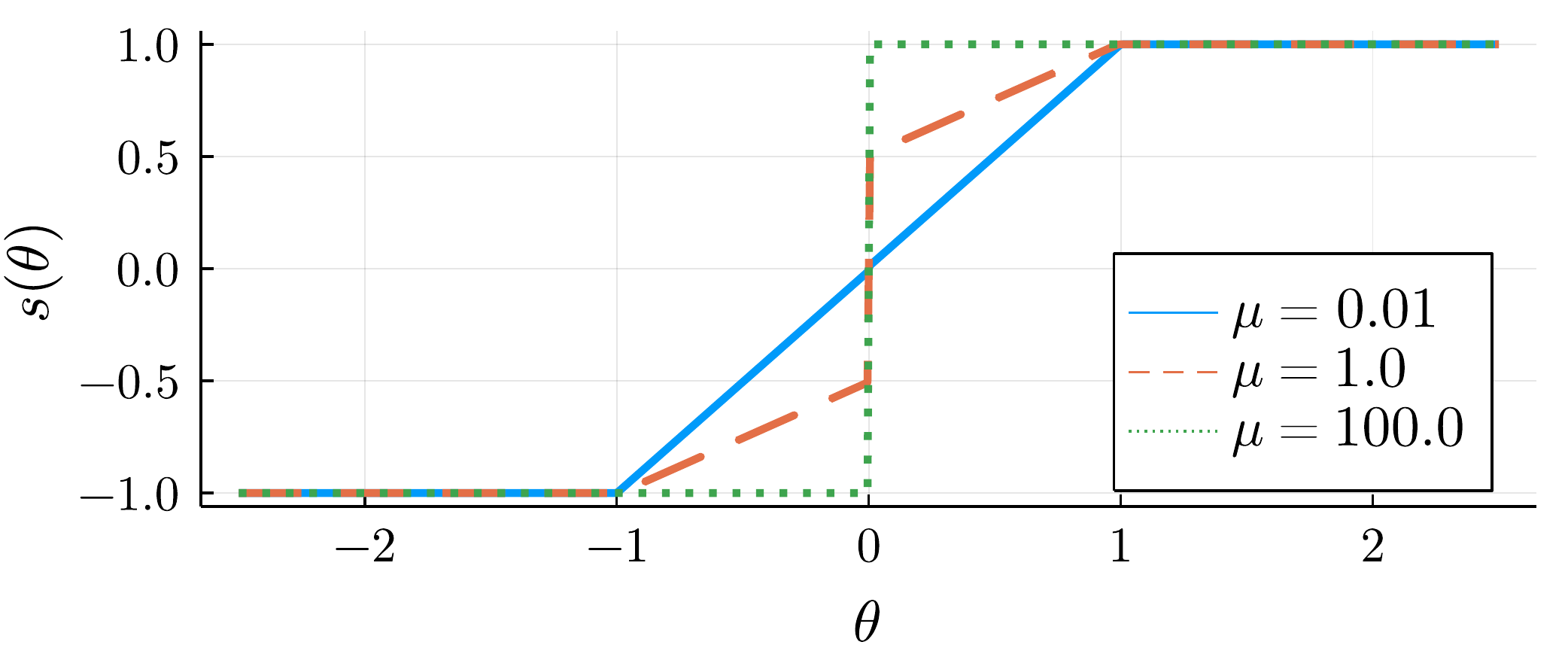}
  \caption{The graph of the mapping $w^*=s(\theta)$ given
    in~\eqref{eq:our_s_mapping} for $\alpha=1/100$ and three different values of $\mu$.}
  \label{fig:activation}
\end{figure}

\begin{table*}[htb]
\centering
 \footnotesize
\begin{tabular}{|l|cc|cc|c|}
\hline
\multirow{2}{*}{Implementation} & \multicolumn{2}{c|}{CIFAR-10} & \multicolumn{2}{c|}{CIFAR-100} & \multicolumn{1}{c|}{TinyImageNet} \\
                    & VGG-16       & ResNet-18      & VGG-16       & ResNet-18       & ResNet-18 \\ \hline
Full-precision ($\dagger$)       & 93.33        & 94.84          & 71.50        & 76.31           & 58.35 \\ \hline
BinaryConnect (*)      & 89.75$\pm$0.26    & 91.92$\pm$0.23        & 54.61$\pm$2.37    & 68.67$\pm$0.7   & - \\
BinaryConnect ($\dagger$)      & 89.04        & 91.64          & 59.13        & 72.14           & 49.65 \\
ProxQuant($\dagger$)            & 90.11        & 92.32          & 55.10        & 68.35           & 49.97 \\
PMF($\dagger$)                  & 91.40        & 93.24          & 64.71        & 71.56           & 51.52 \\
\hline
MD-softmax ($\dagger$)         & 90.47        & 91.28          & 56.25        & 68.49           & 46.52 \\
MD-softmax-s ($\dagger$)         & 91.30        & 93.28          & 63.97        & 72.18           & 51.81 \\
MD-softmax-s (*)        & 83.69$\pm$0.33                & 91.56$\pm$0.14                      & 48.23$\pm$0.55                 & 68.35$\pm$0.96               & - \\
\hline
MD-tanh ($\dagger$)            & 91.64        & 92.27          & 61.31        & 72.13           & 54.62 \\
MD-tanh-s ($\dagger$)        & 91.53        & 93.18          & 61.69        & 72.18           & 52.32 \\
MD-tanh-s (*)           & 90.22$\pm$0.24    & 91.41$\pm$0.11      & 60.14$\pm$0.58      & 66.38$\pm$0.26  & - \\
\hline
BayesBiNN (*)          & 90.68$\pm$0.07   & 92.28$\pm$0.09 & 65.92$\pm$0.18 & 70.33$\pm$0.25 & 54.22  \\ \hline
AdaSTE (no annealing) (*)  & {92.16}$\pm$0.16        & 93.96$\pm$0.14          & 68.46$\pm$0.18        & 73.90$\pm$0.20               &53.49                             \\
AdaSTE (with annealing) (*)  & \textbf{92.37}$\pm$0.09        & \textbf{94.11}$\pm$0.08           & \textbf{69.28}$\pm$0.17         & \textbf{75.03}$\pm$0.35               &\textbf{54.92}                             \\
\hline
\end{tabular}
\caption{Classification accuracy for different methods. (*) indicates that experiments have been run 5 times using different random seeds (except for TinyImageNet). ($\dagger$) indicates that results are obtained from the numbers reported by~\cite{Ajanthan2021}.}
\label{tab:acc}
\end{table*}

\subsection{Adaptive choice for $\beta$}
\label{sec:adaptive_beta}

As indicated in Section~\ref{sec:updating_theta}, we can steer the modified
chain rule by selecting $\beta_j>0$ appropriately in order to determine a
suitable descent direction. Note that each element $\theta_j$ in the vector of
parameters $\theta$ has its own value for $\beta_j$. Below we describe how
$\beta_j$ is chosen when $\alpha$ and $\mu$ satisfy $\mu\alpha\ge 1$.
In this setting we always have $w_j^*=\sgn(\theta_j)\in\{-1,1\}$ and
$\hat w_j=\vec\sgn(\theta_j-\beta_j\ell'_j(w^*))\in\{-1,1\}$ (we ignore the
theoretical possibility of $\theta_j=0$ or
$\theta_j-\beta_j\ell_j'(w^*)=0$). Our aim is to select $\beta_j>0$ such that
the slope induced by backward differences,
$\frac{1}{\beta_j}(w_j^*-\hat w_j)$, is as close to $\ell_j'(w^*)$ as
possible. In the following we abbreviate $\ell'(w^*)$ to $\ell'$.

Since $\sgn$ is an increasing step-function with derivative being zero almost
everywhere, its finite difference approximation
\begin{align}
  \tfrac{1}{\beta_j} \big( w_j^*-\hat w_j \big) = \tfrac{1}{\beta_j} \big( \sgn(\theta_j)-\sgn(\theta_j-\beta_j\ell'_j) \big)
\end{align}
lies either in the interval $[0,s_{\max}]$ or in $[-s_{\max}, 0]$ for a
suitable $s_{\max}\ge 0$ (which is dependent on $\theta_j$ and
$\ell'_j$). In particular, if $\theta_j \ell'_j \le 0$, then
$\sgn(\theta_j)=\sgn(\theta_j-\beta_j\ell'_j)$ for all $\beta_j\ge 0$ and
$s_{\max}=0$.  On the other hand, if $\theta_j \ell'_j > 0$, then
$\sgn(\theta_j-\beta_j\ell'_j)\ne \sgn(\theta_j)$ for
$\beta_j > \theta_j/\ell'_j$ and therefore
\begin{align}
  \sup_{\beta_j>\theta_j/\ell'_j} \frac{|w_j^*-\hat w_j|}{\beta_j} = \frac{2\ell'_j}{\theta_j}.
\end{align}
If $\theta_j$ is close to~0, then the r.h.s.\ may grow arbitrarily large
(reflecting the non-existence of the derivative of $\sgn$ at~0).  Assuming
that $(w_j^*-\hat w_j)/\beta_j$ should maximally behave like a
straight-through estimator (i.e.\
$|w_j^*-\hat w_j|/\beta_j\le |\ell'_j|$, which also can be seen as a form of
gradient clipping), we choose
\begin{align}
    \beta_j = \tfrac{1}{|\ell'_j|} \max\{ 2, |\theta_j| \}
\end{align}
in order to guarantee that
\begin{align}
  \tfrac{1}{\beta_j} |w_j^*-\hat w_j| \le \tfrac{2}{\beta_j} \le \tfrac{2|\ell'_j|}{2} = |\ell'_j|.
\end{align}
Overall, we obtain the following simple rule to assign each $\beta_j$ for
given $\theta$ and $\ell'$:
\begin{align}
  \label{eq:beta_choice}
  \beta_j \gets
  \begin{cases}
    \tfrac{1}{|\ell'_j|} \max\{ 2, |\theta_j| \} & \text{if } \theta_j \ell'_j > 0 \\
    1 & \text{otherwise.}
  \end{cases}
\end{align}
The choice of $\beta_j=1$ in the alternative case is arbitrary, since
$(w_j^*-\hat w_j)/\beta=0$ for all values $\beta>0$. Observe that the
assignment of $\beta_j$ in~\eqref{eq:beta_choice} selectively converts
$(w_j^*-\hat w_j)/\beta_j$ into a scaled straight-through estimator whenever
$\theta_j \ell'_j > 0$, otherwise the effective gradient used to update
$\theta_j$ is zero (in agreement with the chain rule).

In the appendix we discuss the setting when
$\mu\alpha<1$, which yields in certain cases different expressions for
$\beta_j$. Nevertheless, we use~\eqref{eq:beta_choice} in all our experiments.

\section{Experimental Results}
\label{sec:results}

In this section, we show several experimental results to validate the performance of our proposed method and compare it against existing algorithms that achieve state-of-the-art performance for our particular problem settings. As mentioned above, we only consider the training of networks with fully binarized weights and real-valued activations. 

Following previous works~\cite{Ajanthan2021,Bai2019,meng2020}, we use classification as the main task throughout our experiments.  In particular, we evaluate the performance of the algorithms on the two network architectures: ResNet-18 and VGG16. The networks are trained and evaluated on the CIFAR10, CIFAR100 and TinyImageNet200~\cite{le2015tiny} datasets.
We compare our algorithm against state-of-the-art approaches, including BinaryConnect (BC)~\cite{Courbariaux2016}, ProxQuant (PQ)~\cite{Bai2019}, Proximal Mean-Field (PMF)~\cite{Ajanthan2019}, BayesBiNN~\cite{meng2020}, and several variants of Mirror Descent (MD)~\cite{Ajanthan2021}.
We employ the same standard data augmentations and normalization as employed by the methods we compare against (please refer to our appendix for more details about the experimental setup). Our method is implemented in Pytorch and is developed based on the software framework released by BayesBiNN's authors\footnote{\url{https://github.com/team-approx-bayes/BayesBiNN}} (more details regarding our implementation can be found in the appendix). 

\subsection{Classification Accuracy}
In Table~\ref{tab:acc}, we report the best testing accuracy obtained by the considered methods. For PQ, PMF, the unstable versions of MD as well as for full-precision reference networks, we use the best results report in~\cite{Ajanthan2021}. For BC, the stable variants of MD (i.e. MD-softmax-s and MD-tanh-s), we reproduce the results by running the source code released by the authors\footnote{\url{https://github.com/kartikgupta-at-anu/md-bnn}} (using the default recommended hyper-parameters) for $5$ different random initializations, and reporting the mean and standard deviation obtained from these runs. The same strategy is also applied to BayesBiNN (hyper-parameters for BayesBiNN can be found in the appendix), except for the TinyImageNet dataset where we only report results for a single run (due to longer training time of TinyImageNet). We report the results for our method using two settings:
\begin{itemize}
    \item Without annealing: we set $\alpha=0.01$ and fix $\mu = \frac{1}{\alpha}$ throughout training.
    \item With annealing: we also use $\alpha=0.01$ and set the initial value $\mu$ to $\mu^{(0)}=1.0$, then increase $\mu$ after each epoch by a factor of $\gamma$, i.e. $\mu^{(t)} \gets \gamma\mu^{(t-1)}$. $\gamma$ is chosen such that $\mu$ reaches $1/\alpha$ after $\approx 200$ epochs.
\end{itemize}
The impact of the choice of $\mu$ on the shape of $\vec{s}(\theta)$ is illustrated in Fig.~\ref{fig:activation}.

\begin{figure*}[ht]
    \centering
    \includegraphics[width=0.49\textwidth]{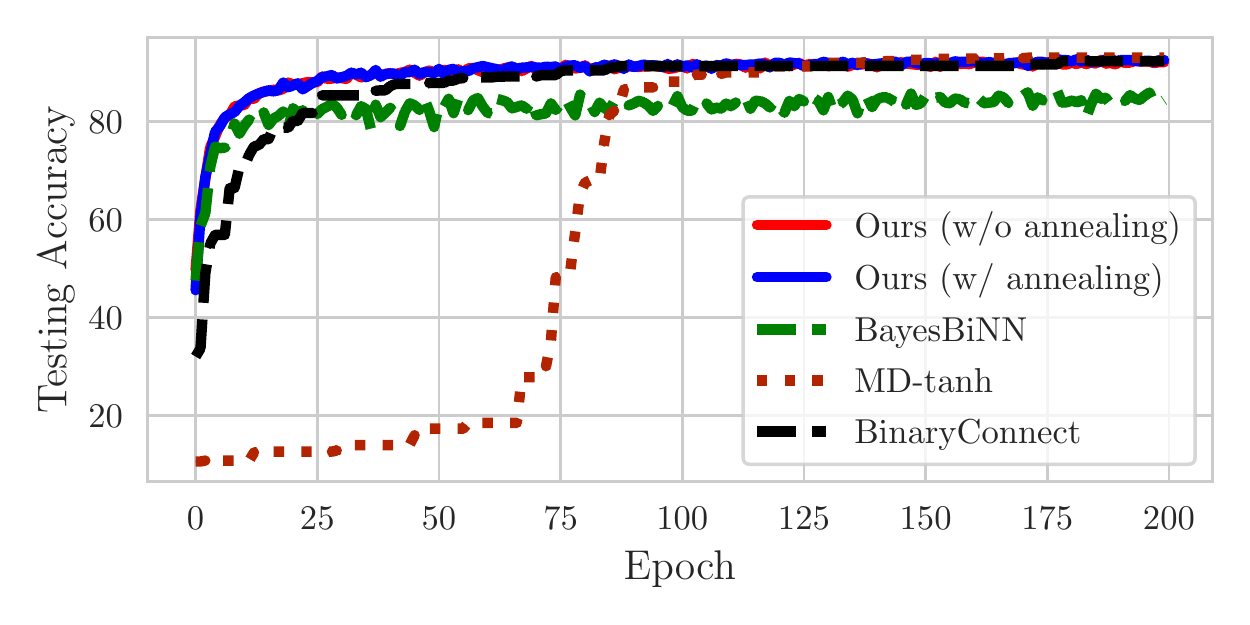}
    \includegraphics[width=0.49\textwidth]{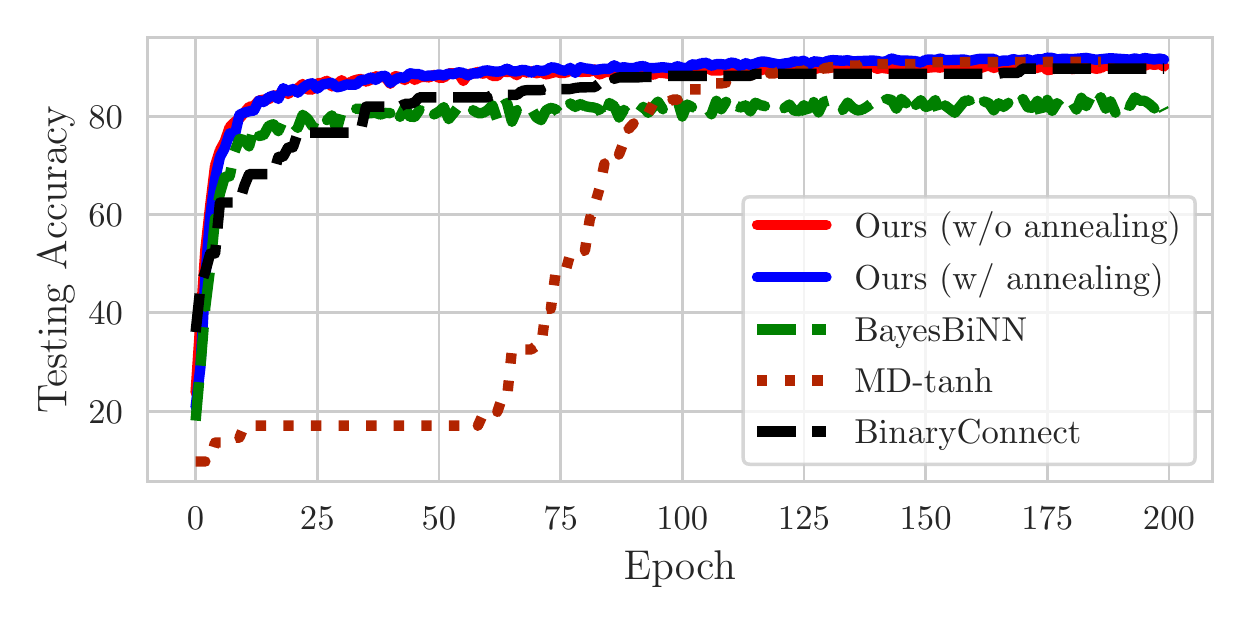}
    \caption{Testing accuracy achieved by the methods for the first 200 epochs with ResNet-18 (left) VGG16 (right) for CIFAR10 dataset (plots for CIFAR100 can be found in the appendix).   }
    \label{fig:acc_200}
\end{figure*}

\begin{figure*}[ht]
    \centering
    \includegraphics[width=0.49\textwidth]{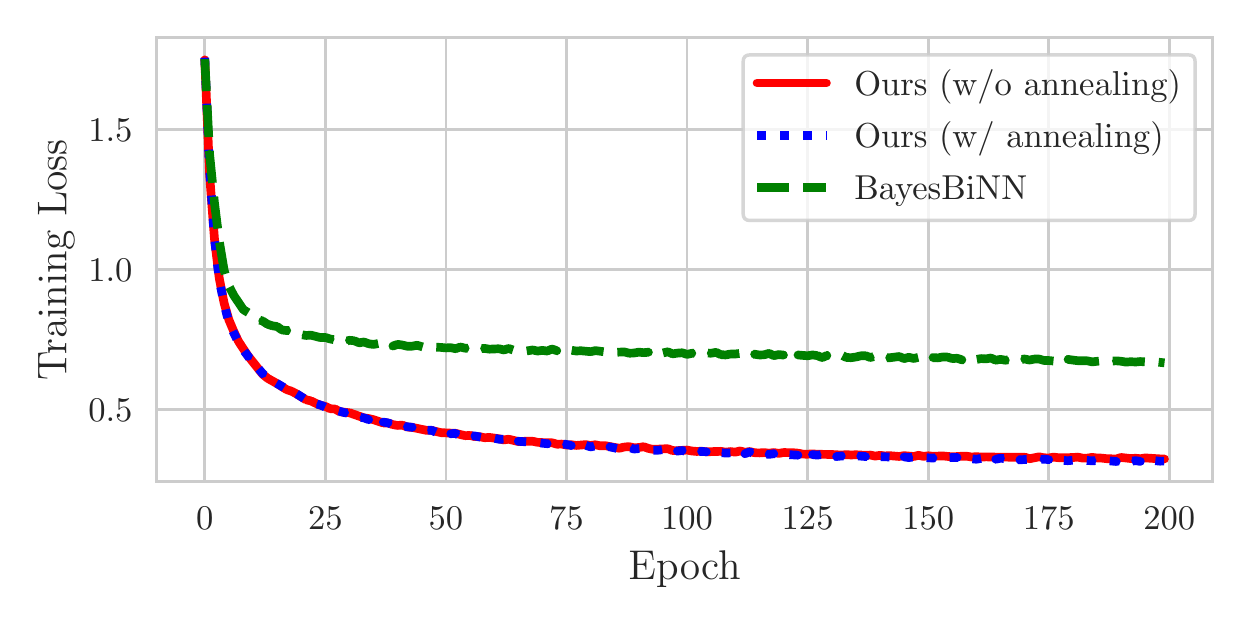}
    \includegraphics[width=0.49\textwidth]{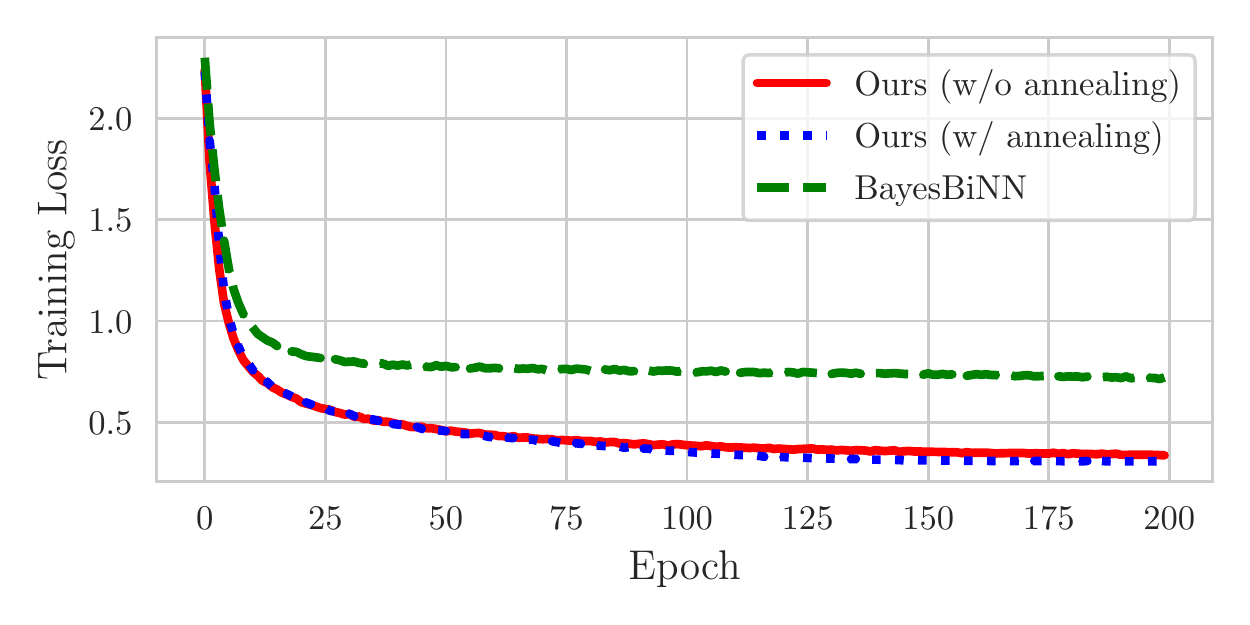}
    \caption{Training loss of the methods for the first 200 epochs with ResNet-18 (left) and VGG16 (right) on the CIFAR10 dataset (see appendix for plots of CIFAR100 dataset) .   }
    \label{fig:loss_200}
\end{figure*}

Table~\ref{tab:acc} demonstrates that our proposed algorithm achieves state-of-the-art results. Note that we achieve highly competitive results even without annealing $\mu$ (although annealing improves the test accuracy slightly but consistently).
Hence, we conclude that AdaSTE without annealing (and therefore no additional hyper-parameters) can be used as direct replacement for BinaryConnect.
Note that we report all results after training for $500$ epochs. In the appendix, we will show that both BayesBiNN and AdaSTE yield even higher accuracy if the models are trained for higher number of epochs.

\subsection{Evolution of Testing Accuracy and Training Losses}

We further investigate the behavior of the algorithms during training. In particular, we are interested in the evolution of training losses and testing accuracy, since these quantities are---in addition to the achieved test accuracy---of practical interest.

In Fig.~\ref{fig:acc_200}, we plot the testing accuracy obtained by our method in comparison with BC, MD (using the tanh mapping), and BayesBiNN for the first 200 epochs. For our method, we show the performance for both settings with and without annealing (as described earlier). To obtain the plots for MD and BayesBiNN, we use the code provided by the authors with the default recommended hyper parameters. For BC, we use the implementation provided by MD authors. As can be observed, AdaSTE quickly reaches very high test accuracy compared to other approaches. The MD-tanh approach (using the recommended annealing schedule from the authors~\cite{Ajanthan2021}) only reaches satisfactory accuracy after approximately $100$ epochs. We also try starting MD-tanh with a larger annealing parameter (i.e. the $\beta$ hyper-parameter in~\cite{Ajanthan2021}), but that yields very poor results (see the appendix for more details). AdaSTE, on the other hand, is quite insensitive to the annealing details, and yields competitive results even without annealing.

Fig.~\ref{fig:loss_200} depicts the training loss of our methods compared to BayesBiNN. We choose to compare AdaSTE against our main competitor, BayesBiNN, as we have full control of the source code to assure that both methods are initialized with the same starting points.  As can be seen, our method quickly reduces the training loss, while BayesBiNN takes longer for the training loss to converge.  
Note that BayesBiNN leverages the reparametrization trick and relies therefore
on weights sampled from respective distributions at training time. In that
sense AdaSTE is a purely deterministic algorithm, and the only source of
stochasticity is the sampled mini-batches. This might be a factor explaining
AdaSTE's faster reduction of the training loss.

\section{Discussion and Conclusion}

In this work we propose AdaSTE, an easy-to-implement replacement for the
straight-through gradient estimator, and we demonstrate its benefits for
training DNNs with strictly binary weights.  One clear limitation in this work
is, that we focus on the binary weight but real-valued activations scenario,
which is a highly useful setting, but still prevents low-level implementations
using only $\operatorname{xor}$ and bit count operations.  Extending AdaSTE to
binary activations seems straightforward, but will be more difficult to
justify theoretically, and we expect training to be more challenging in
practice.
One obvious further shortcoming is our restriction to purely binary
quantization levels, in particular to the set $\{+1,-1\}$.  Generalizing the
approach to arbitrary quantization levels can be done in several ways, e.g.\
by extending the W-shaped cost $E$ in~\eqref{eq:our_E} to more minima or by
moving to higher dimensions (e.g.~by modeling parameters in the probability
simplex).

Since weight quantization is one option to regulate the Lipschitz property of a DNNs' forward mapping (and also
its expressive power), the impact of
weight quantization~\cite{song2020improving,duncan2020relative} (and more generally
DNN model compression~\cite{gui2019model,ye2019adversarial}) on adversarial robustness has been recently
explored. Hence, combining our adaptive straight-through gradient estimator
with adversarial training is one direction of future work.

{
  \small
  \bibliographystyle{plain}
  \bibliography{literature}
}

\appendix

\section{A Mirror Descent Interpretation of AdaSTE}

In this section we establish a connection between AdaSTE and mirror descent
with a data-adaptive and varying metric. Since the update in AdaSTE is applied
element-wise, we focus on the update of $\theta_j$ (a scalar) in the
following. For brevity of notation we drop the subscript $j$.

We consider using a ``partial'' chain rule as follows. Let the target forward
mapping be the composition of $s_1$ and $s_2$, i.e.\ $s=s_2\circ s_1$. Then
the AdaSTE update step is abstractly given by
\begin{align}
  \theta^{(t+1)} &\gets \theta^{(t)} - \eta \ell'(s_2(s_1(\theta^{(t)}))) s_2'(s_1(\theta^{(t)})).
\end{align}
Observe that only one step of the chain rule is applied on $\ell$ as
$s_1'$ is not used. We introduce an ``intermediate'' weight
$u=s_1(\theta)$, and therefore
$w=s_2(u)=s_2(s_1(\theta))=s(\theta)$. Expressing the above update step in
$u$ yields
\begin{align}
  s_1^{-1}(u^{(t+1)}) &\gets s_1^{-1}(u^{(t)}) - \eta \ell'(s_2(u^{(t)})) s_2'(u^{(t)}),
\end{align}
and identifying $s_1^{-1}$ with the mirror map $\nabla\Phi$ results eventually in
\begin{align}
  u^{(t+1)} &= \arg\min_u \tfrac{1}{\eta} D_\Phi(u\|u^{(t)}) + \ell'(s_2(u^{(t)})) s_2'(u^{(t)}) \nonumber \\
  &= \arg\min_u \tfrac{1}{\eta} D_\Phi(u\|u^{(t)}) + \tfrac{d}{du} \ell(s_2(u)) \big|_{u=u^{(t)}} .
\end{align}
Now the question is whether there exist mappings $s_1$ and $s_2$ such that
\begin{align}
  s_2(s_1(\theta)) = s(\theta) & & s_2'(u) = s'(s_1^{-1}(u) - h),
\end{align}
where will be chosen as $h=\beta\ell'$ in AdaSTE. The first relation yields
\begin{align}
  s_1(\theta) = s_2^{-1}(s(\theta)) \;\text{ and } \; s_1^{-1}(u) = s^{-1}(s_2(u)).
\end{align}
Hence, the second condition above is equivalent to
\begin{align}
  s_2'(u) = s'(s_1^{-1}(u) - h) = s'\big(s^{-1}(s_2(u)) - h \big) \nonumber.
\end{align}
By expressing this relation in terms of $\theta$ we obtain
\begin{align}
  &s_2'(s_1(\theta)) = s'(\theta-h) \iff s_2'(s_2^{-1}(s(\theta))) = s'(\theta-h) \nonumber \\
  &\iff \frac{1}{(s_2^{-1})'(s(\theta))} = s'(\theta-h) \nonumber \\
  &\iff (s_2^{-1})'(w) = \frac{1}{s'(s^{-1}(w)-h)} \nonumber.
\end{align}
Consequently, $s_2^{-1}$ can be determined by solving
\begin{align}
  s_2^{-1}(w) = \int_{w_0}^w \frac{1}{s'(s^{-1}(\omega)-h)}\,d\omega.
\end{align}
If $h=0$, then $s_2^{-1}=s^{-1}$ (and therefore $s_1=\mathrm{id}$) is a valid
solution. For $h\ne 0$, there is sometimes a closed-form expression for
$s_2^{-1}$. We consider $s=\tanh$, i.e.
\begin{align}
  s(\theta) = \frac{e^\theta - e^{-\theta}}{e^\theta + e^{-\theta}}
  = \frac{e^{2\theta}-1}{e^{2\theta}+1} & & s'(\theta) = \frac{4e^{2\theta}}{(e^{2\theta}+1)^2}.
\end{align}
With this choice we obtain (via a computer algebra system)
\begin{align}
  (s_2^{-1})'(w) &= \frac{1}{s'(s^{-1}(w)-h)} \nonumber \\
  &= \frac{e^{-2h}\big( (e^{2h} - 1)w - e^{2h}-1 \big)^2}{4(1-w^2)} \nonumber \\
  &= \frac{\big( (e^h - e^{-h})w - e^h-e^{-h} \big)^2}{4(1-w^2)}.
\end{align}
Now the following relation holds,
\begin{align}
  &\int \frac{(aw+b)^2}{4(1-w^2)}\,dw \nonumber \\
  &\doteq \tfrac{1}{8} \left( -2 a^2 w - (a + b)^2 \log(1 - w) + (a - b)^2 \log(1 + w) \right). \nonumber
\end{align}
Plugging in the values $a=e^h - e^{-h}$ and $b=-e^h - e^{-h}$ (and therefore
$a+b=-2e^{-h}$ and $a-b=2e^h$) results in
\begin{align}
  &s_2^{-1}(w) \nonumber \\
  &= \tfrac{1}{8} \left( -2 (e^h \!-\! e^{-h})^2 w - 4e^{-2h} \log(1 \!-\! w) + 4e^{2h} \log(1 \!+\! w) \right) \nonumber \\
  &= \tfrac{1}{2} \left( e^{2h} \log(1 + w) - e^{-2h} \log(1 - w) \right) \nonumber \\
  &- \tfrac{1}{4} (e^h - e^{-h})^2 w.
\end{align}
As expected, for $h=0$ we obtain $\tanh^{-1}$, and for $h\ne 0$ this mapping
skews $\tanh^{-1}$. The important property is, that $s_2$ is strictly monotone
since $s_2'(s_1(\theta))=s'(\theta-h)>0$. We can recover $s_1$ via
$s_1(x)=s_2^{-1}(s(\theta))$, but that seems to be a non-interpretable
expression in this case.

\section{AdaSTE: the case $\mu\alpha<1$}

As in the previous section we focus on one scalar weight
$\theta_j$/$w_j$ and omit the subscript $j$ in the following. We know that the
actual weight $w$ is obtained via
\begin{align}
  w^* &= \Pi_{[-1,1]}\left( \frac{\theta + \mu(1+\alpha)\sgn(\theta)}{1+\mu} \right) \nonumber \\
  \hat w &= \Pi_{[-1,1]}\left( \frac{\tilde\theta + \mu(1+\alpha)\sgn(\tilde\theta)}{1+\mu} \right),
\end{align}
where $\tilde\theta = \theta-\beta\ell'$. We focus on $\theta<0$, since the
case $\theta>0$ is symmetric. Hence,
\begin{align}
  w^* = \begin{cases}
    -1 & \text{if } \theta \le -1 + \mu\alpha \\
    \frac{\theta-\mu(1+\alpha)}{1+\mu} & \text{ if } \theta \in (-1+\mu\alpha, 0)
  \end{cases}
\end{align}
and
\begin{align}
  \hat w = \begin{cases}
    -1 & \text{if } \tilde\theta \le -1 + \mu\alpha \\
    \frac{\tilde\theta-\mu(1+\alpha)}{1+\mu} & \text{ if } \tilde\theta \in (-1+\mu\alpha, 0)
  \end{cases}.
\end{align}
We are now interested in values for $\beta>0$ maximizing
$|\hat w-w^*|/\beta$. We assume that $\mu\alpha<1$, since the simpler setting
$\mu\alpha\ge 1$ was discussed in the main text.

\paragraph{Case $\ell' > 0$:}
We have $\tilde\theta = \theta-\beta\ell' < \theta$ for all $\beta>0$. Since
$\hat w$ will be clamped at $-1$ for sufficiently large $\beta>0$, the
solution for $\beta$ satisfies
\begin{align}
  \theta - \beta\ell' \in (-1 + \mu\alpha, 0).
\end{align}
If $\theta \le -1+\mu\alpha$, then we have $w^*=\hat w = -1$ for all choices
of $\beta$, and therefore $(\hat w-w^*)/\beta=0$ regardless of
$\beta$. Thus, we assume that $\theta>-1+\mu\alpha$ and therefore
$w^*>-1$. For $\beta$ constrained as above, we have
\begin{align}
  \frac{\hat w - w^*}{\beta} &= \frac{1}{\beta} \cdot \frac{\theta - \beta\ell' - \mu(1+\alpha) - (\theta-\mu(1+\alpha))}{1+\mu} \nonumber \\
  {} &= \frac{1}{\beta} \cdot \frac{\beta\ell'}{1+\mu} = \frac{\ell'}{1+\mu} \nonumber,
\end{align}
which is independent of the exact value of $\beta$ as long it is in the
allowed range,
\begin{align}
  \beta \in \tfrac{1}{\ell'} (\theta, \theta + 1 - \mu\alpha) \cap \mathbb R_{\ge 0}.
\end{align}
We can set $\beta$ as follows,
\begin{align}
  \beta = \min\left\{ \beta_{\max}, \frac{\theta+1-\mu\alpha}{\ell'} \right\} \nonumber
\end{align}
and the error signal is given by $(\hat w - w^*)/\beta = \ell'/(1+\mu)$.

\paragraph{Case $\ell'<0$:}
This means that $\tilde\theta > \theta$ for $\beta>0$. By inspecting the
piecewise linear (and monotonically increasing) mapping
$\theta\mapsto w^*$ we identify two relevant choices for $\beta$:
$\beta_1$ as the smallest $\beta$ such that $\hat w$ is clamped at
$+1$, and $\beta_0$ as the smallest $\beta$ such that $\hat w$ is
positive. Note that $\tilde\theta$ is clamped at $+1$ whenever
$\tilde\theta > 1-\mu\alpha$. Therefore the defining constraints for
$\beta_1$ and $\beta_0$ are given by
\begin{align}
  \theta - \beta_1\ell' = 1 - \mu\alpha & & \theta - \beta_0\ell' = 0^+ \nonumber,
\end{align}
i.e.\ $\beta_1 = (\theta-1+\mu\alpha)/\ell'$ and
$\beta_0 = \theta/\ell'$ (and $\beta_1 > \beta_0$ by construction). If
$\tilde\theta=0^+$, then $\hat w = \mu(1+\alpha)/(1+\mu)$. Consequently,
\begin{align}
  \frac{\hat w_1 - w^*}{\beta_1} &= \frac{\ell'}{\theta-1+\mu\alpha}
                                  \left( 1 - \max\left\{ -1, \frac{\theta - \mu(1+\alpha)}{1+\mu} \right\} \right) \nonumber \\
  \frac{\hat w_0 - w^*}{\beta_0} &= \frac{\ell'}{\theta}
                                  \left( \frac{\mu(1+\alpha)}{1+\mu} - \max\left\{ -1, \frac{\theta - \mu(1+\alpha)}{1+\mu} \right\} \right) \nonumber.
\end{align}
If $\theta\le -1+\mu\alpha$ such that $w^*=-1$, then these expressions simplify to
\begin{align}
  \frac{\hat w_1 - w^*}{\beta_1} &= \frac{2\ell'}{\theta-1+\mu\alpha} > 0 \nonumber \\
  \frac{\hat w_0 - w^*}{\beta_0} &= \frac{\ell'}{\theta} \cdot \frac{\mu+\mu\alpha + 1+\mu}{1+\mu}
                                  = \frac{\ell'(1+2\mu+\mu\alpha)}{(1+\mu)\theta} > 0 \nonumber.
\end{align}
Now $(\hat w_1-w^*)/\beta_1 > (\hat w_0-w^*)/\beta_0$ iff
\begin{align}
  & \frac{2\ell'}{\theta-1+\mu\alpha} > \frac{\ell'(1+2\mu+\mu\alpha)}{(1+\mu)\theta} \nonumber \\
  & \iff \frac{2}{\theta-1+\mu\alpha} < \frac{1+2\mu+\mu\alpha}{(1+\mu)\theta} \nonumber \\
  & \iff 2(1+\mu)\theta < (\theta-1+\mu\alpha) (1+2\mu+\mu\alpha) \nonumber \\
  & \iff (1-\mu\alpha) (\theta + 1 + 2\mu + \mu\alpha) < 0 \nonumber \\
  &\iff \theta < -1-2\mu-\mu\alpha \nonumber.
\end{align}
Visual inspection shows that $\beta_0$ a good solution even when $\beta_1$ is the maximizer: $\beta_0$ does not maximize the slope $(\hat w - w^*)/\beta$, but its slope is close to the maximal one.

If $\theta\in(-1+\mu\alpha,0)$, then
$w^*=(\theta-\mu(1+\alpha))/(1+\mu)$ and therefore
\begin{align}
  \frac{\hat w_1 - w^*}{\beta_1} &= \frac{\ell'}{\theta-1+\mu\alpha} \left( 1 - \frac{\theta - \mu(1+\alpha)}{1+\mu} \right) \nonumber \\
                                 & = \frac{\ell'}{\theta-1+\mu\alpha} \cdot \frac{1 + \mu - \theta + \mu(1+\alpha)}{1+\mu} \nonumber \\
  \frac{\hat w_0 - w^*}{\beta_0} &= \frac{\ell'}{\theta} \cdot \frac{\mu(1+\alpha) - \theta + \mu(1+\alpha)}{1+\mu} \nonumber.
\end{align}
$(\hat w_1-w^*)/\beta_1 > (\hat w_0-w^*)/\beta_0$ iff (after dividing both
sides by $1+\mu>0$)
\begin{align}
  & \frac{(1 + 2\mu +\mu\alpha - \theta) \ell'}{\theta-1+\mu\alpha} > \frac{(2\mu(1+\alpha) - \theta) \ell'}{\theta} \nonumber \\
  &\iff \frac{1 + 2\mu +\mu\alpha - \theta}{\theta-1+\mu\alpha} < \frac{2\mu(1+\alpha) - \theta}{\theta} \nonumber \\
  &\iff (1 + 2\mu +\mu\alpha - \theta)\theta < (2\mu(1+\alpha) - \theta)(\theta-1+\mu\alpha) \nonumber \\
  &\iff 2\mu(1-\mu\alpha) (1+\alpha) < 0 \nonumber
\end{align}
The l.h.s.\ is always positive under our assumptions, therefore
$\beta_0=\theta/\ell'$ is the maximizer in this case.

\iftrue
\section{Imagenette Results and Mixup}
In order to further justify if our model also works well on images at higher resolution, we conduct the same experiment on Imagenette dataset~\cite{howard2020fastai} which are sampled from Imagenet~\cite{deng2009imagenet} without being downsampled and consists of 9469 training images and 3925 validation images. Besides, we also notice that mixup~\cite{he2019bag}, a proven effective training trick, is also helpful in further boosting the classification accuracy. As can be seen in Table~\ref{tab:additional_acc}, it is quite obvious that our AdaSTE consistenly outperforms BayesBiNN on both TinyImageNet and Imagenette datasets with and without mixup.

\begin{table}[!htb]
\centering
\begin{tabular}{lcc}
\hline
\multicolumn{1}{|l|}{}                                 & \multicolumn{1}{c|}{\begin{tabular}[c]{@{}c@{}}TinyImageNet\\ ResNet-18\end{tabular}} & \multicolumn{1}{c|}{\begin{tabular}[c]{@{}c@{}}Imagenette\\ ResNet-18\end{tabular}} \\ \hline
\multicolumn{1}{|l|}{BayesBiNN}                        & \multicolumn{1}{c|}{54.22}                                                            & \multicolumn{1}{c|}{78.19}                                                          \\
\multicolumn{1}{|l|}{BayesBinn (mixup)}                & \multicolumn{1}{c|}{55.84}                                                            & \multicolumn{1}{c|}{79.59}                                                          \\ \hline
\multicolumn{1}{|l|}{AdaSTE}          & \multicolumn{1}{c|}{54.92}                                                            & \multicolumn{1}{c|}{79.66}                                                          \\
\multicolumn{1}{|l|}{AdaSTE mixup)} & \multicolumn{1}{c|}{56.11}                                                            & \multicolumn{1}{c|}{80.91}                                                          \\ \hline
                                                       & \multicolumn{1}{l}{}                                                                  & \multicolumn{1}{l}{}                                                               
\end{tabular}
\caption{Classification accuracy for different methods on Tiny Imagenet and Imagenette: Annealing is applied to our model with and without mixup}
\label{tab:additional_acc}
\end{table}

\fi

\section{Implementation Details}
We implemented our AdaSTE algorithm in PyTorch, which is developed based on the framework provided by BayesBiNN. In particular, we used SGD with momentum of $0.9$ for all experiments. 
\begin{itemize}
    \item For CIFAR-10 and CIFAR-100 datasets, we used batch size of $128$ with learning rate of $10^{-5}$.
    \item For TinyImageNet, the chosen batch size was $100$ with the learning rate of  $10^{-6}$.
\end{itemize}
The experimental results for BayesBiNN were produced with the following hyper parameters: 
\begin{itemize}
    \item Batch size: $128$. 
    \item Learning rate: $3 \times 10^{-4}$.
    \item Momentum: $0.9$.
\end{itemize}

\section{CIFAR-100 Results}
Similar to Fig. $3$ and Fig. $4$ in the main text, in Fig.~\ref{fig:acc_200b} and Fig.~\ref{fig:loss_200b}, we also show the test accuracy and training loss versus number of epochs for the CIFAR-100 dataset with ResNet-18 and VGG-16 architectures. The same conclusion can also be drawn, where AdaSTE can quickly achieve very good performance, while it takes longer for other methods to yield high accuracy. This emphasizes the advantage of our method compared to existing approaches.

\begin{figure*}[ht]
    \centering
    \includegraphics[width=0.49\textwidth]{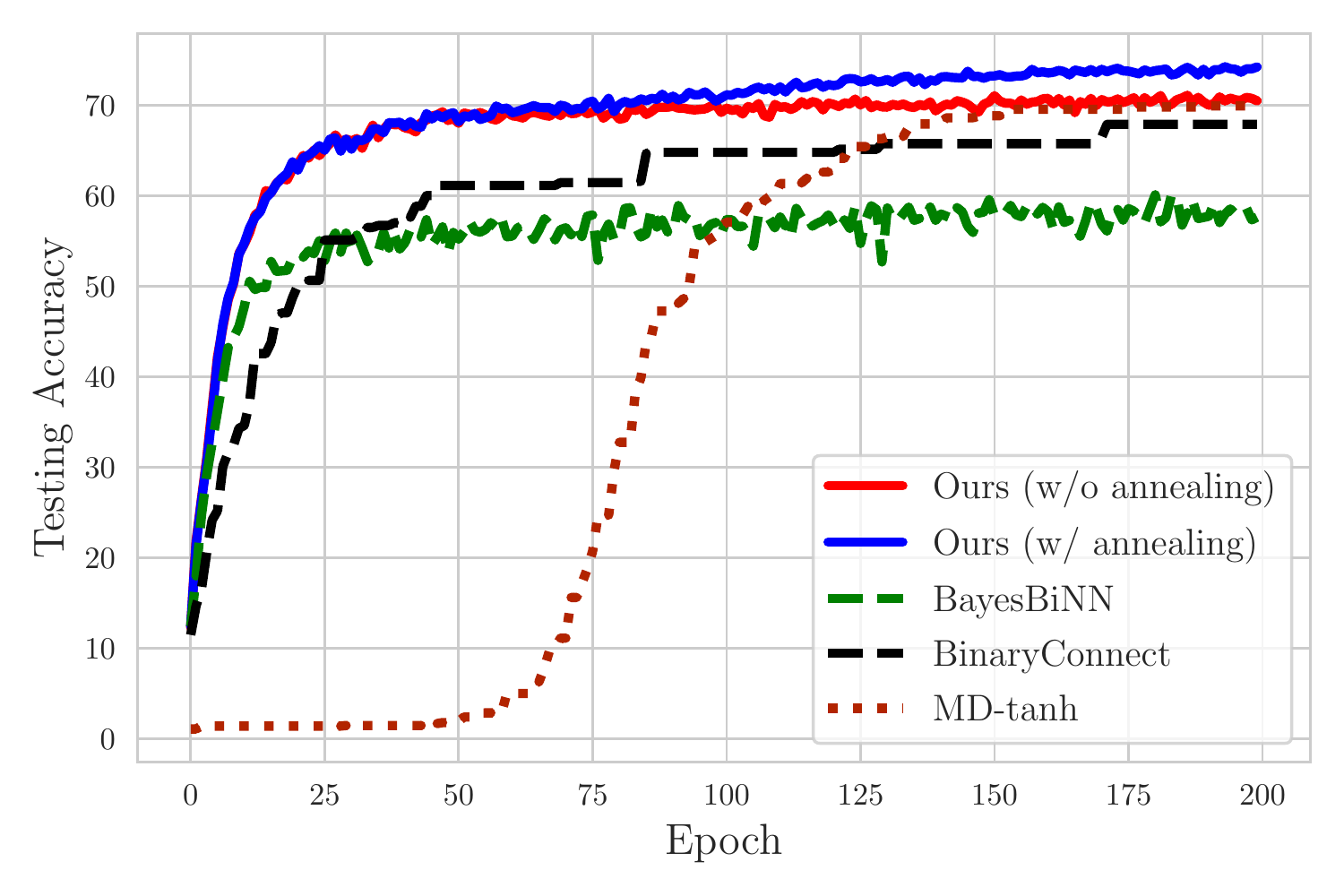}
    \includegraphics[width=0.49\textwidth]{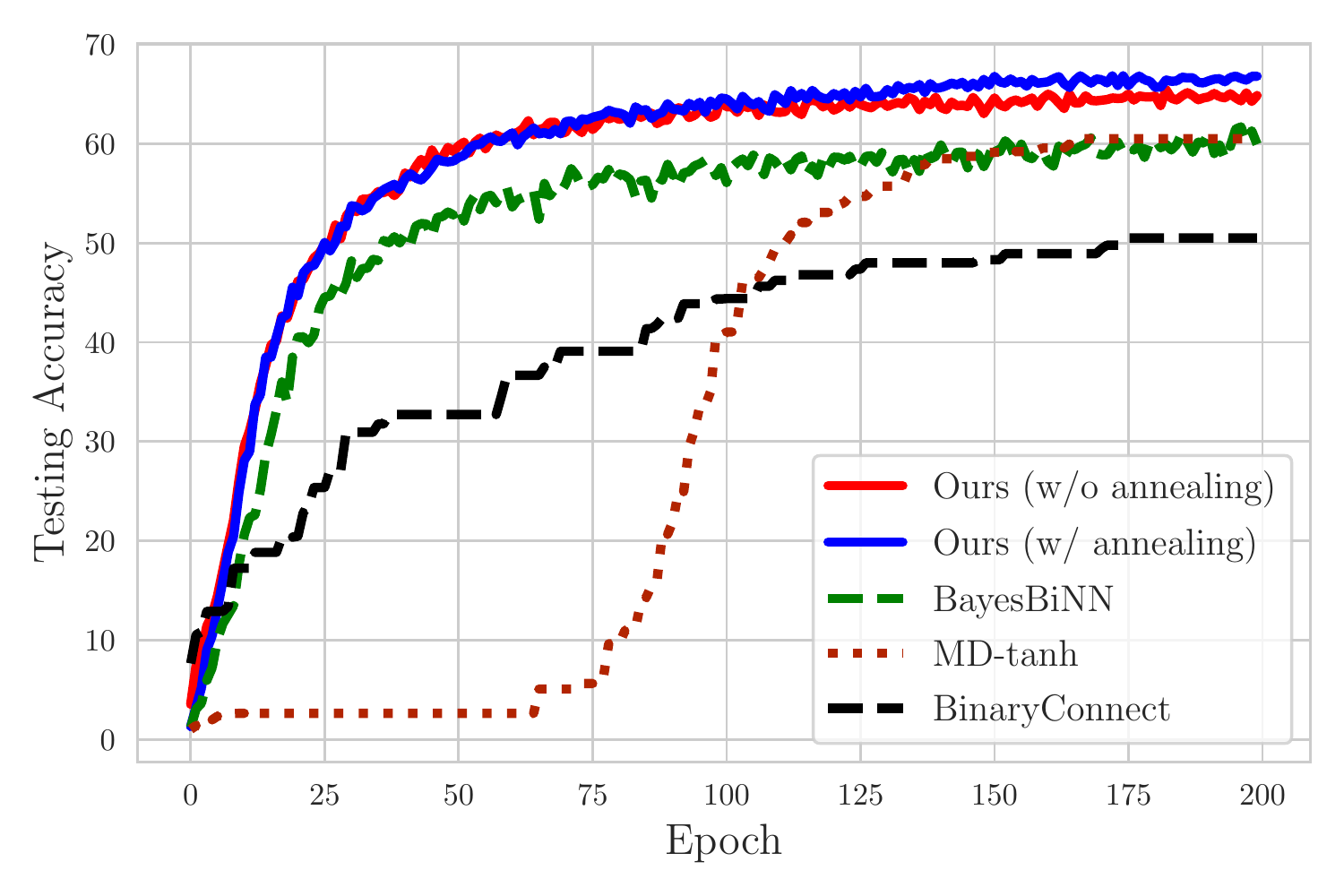}
    \caption{Testing accuracy achieved by the methods for the first 200 epochs with ResNet-18 (left) VGG16 (right) for CIFAR100 dataset. }
    \label{fig:acc_200b}
\end{figure*}

\begin{figure*}[ht]
    \centering
    \includegraphics[width=0.49\textwidth]{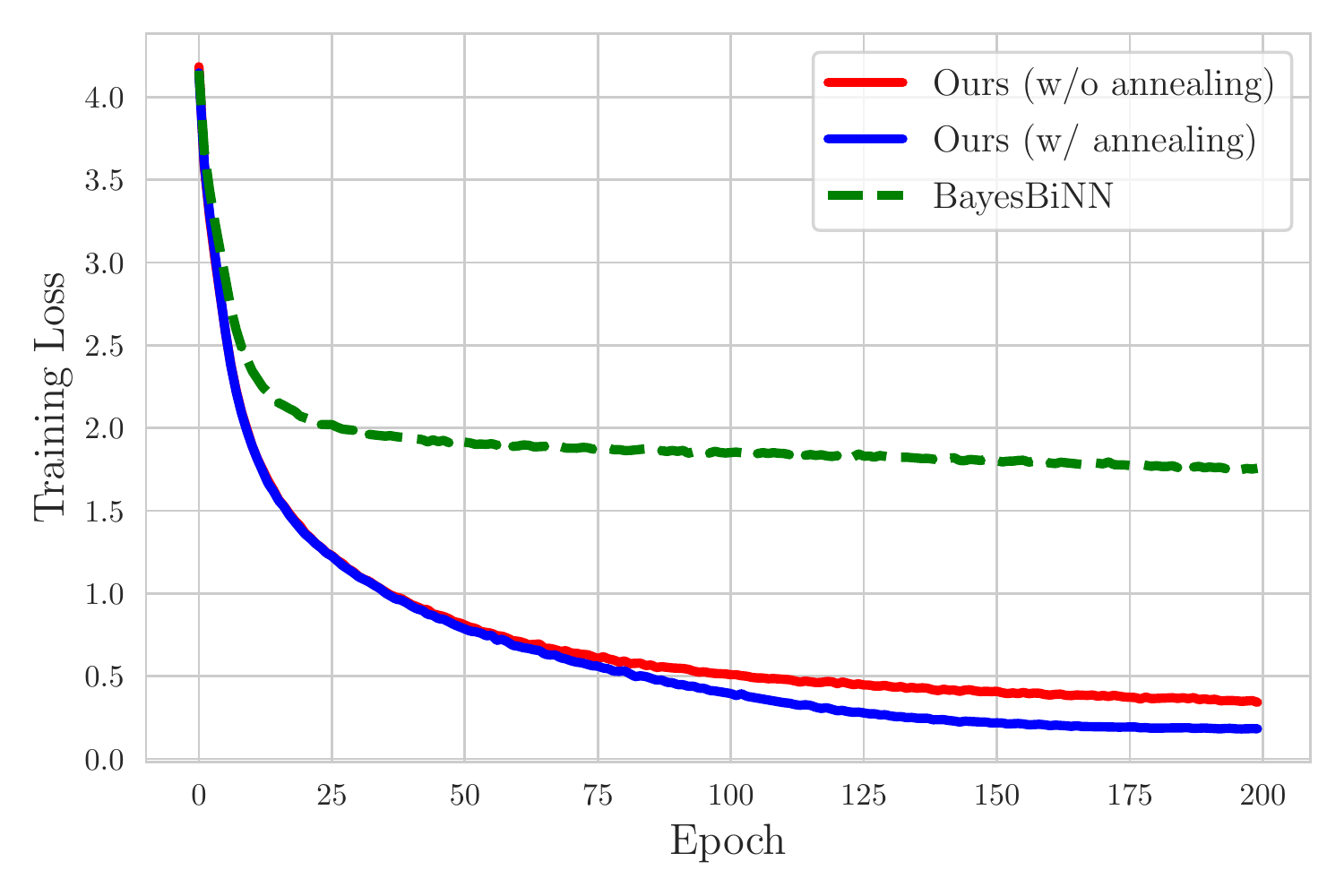}
    \includegraphics[width=0.49\textwidth]{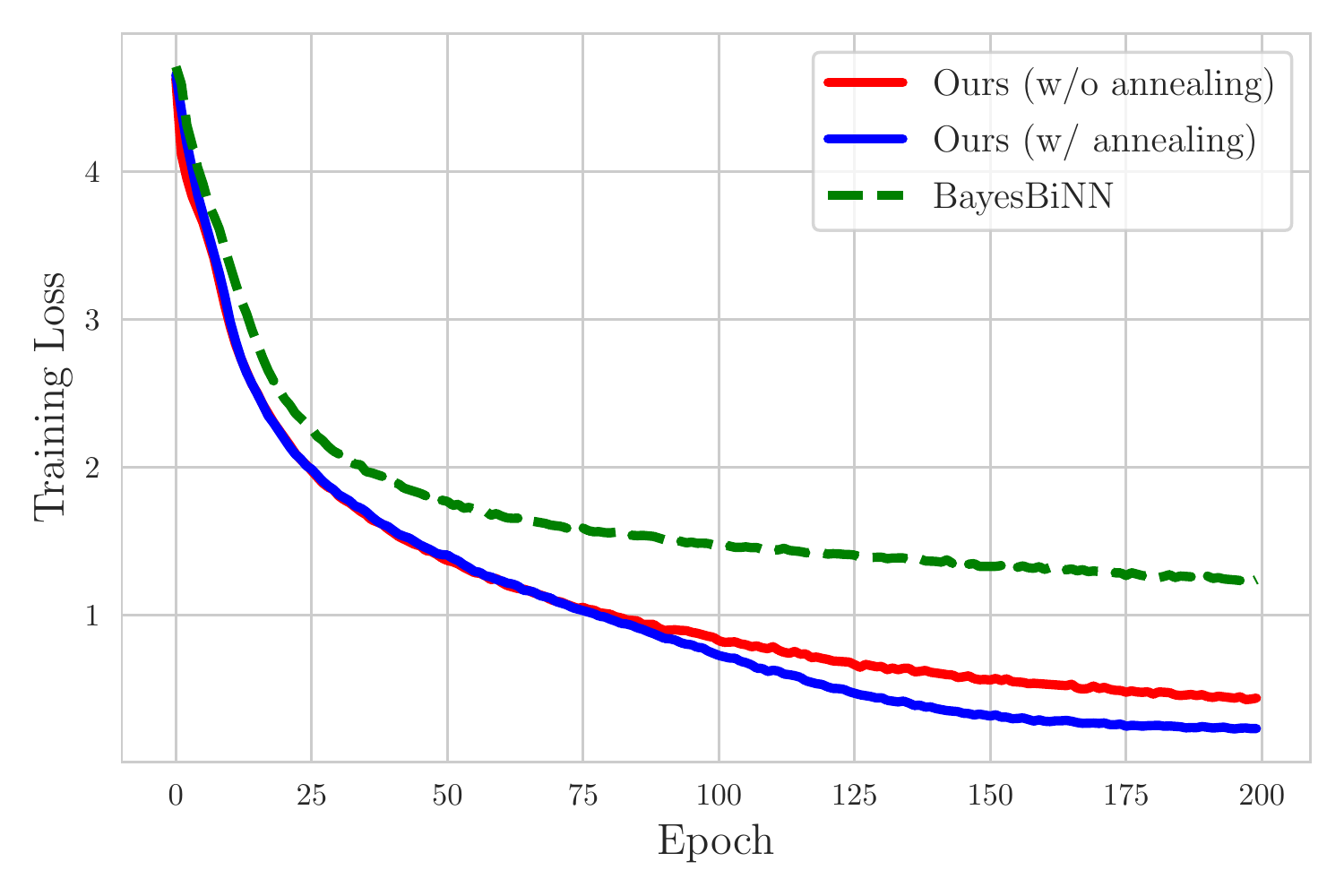}
    \caption{Training loss of the methods for the first 200 epochs with ResNet-18 (left) and VGG16 (right) on the CIFAR100 dataset. }
    \label{fig:loss_200b}
\end{figure*}

\section{Training AdaSTE and BayesBiNN for a larger number of epochs}
In Table $1$ in the main text, we report results obtained after training BayesBiNN and AdaSTE for $500$ epochs. In Fig.~\ref{fig:acc_700b}, we further show the progress of BayesBiNN and AdaSTE after training for $700$ epochs. As can be seen, the performance of both BayesBiNN and AdaSTE can still be improved, and BayesBiNN slowly approaches the performance of AdaSTE.
\begin{figure*}[htb]
    \centering
    \includegraphics[width=0.49\textwidth]{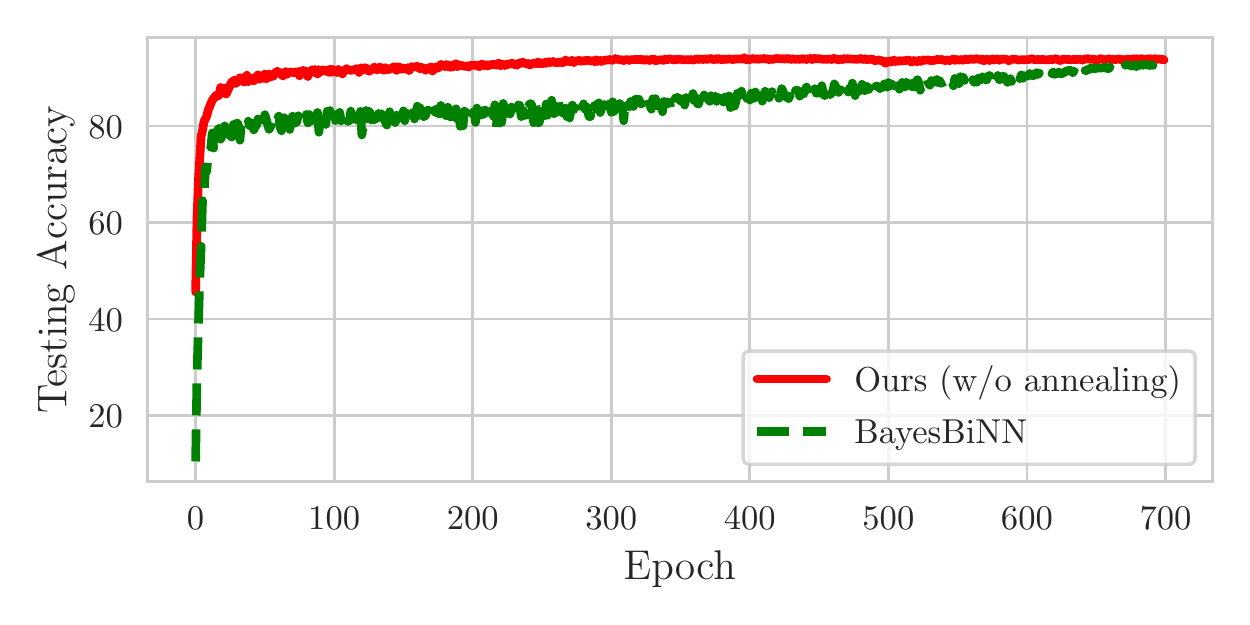}
    \includegraphics[width=0.49\textwidth]{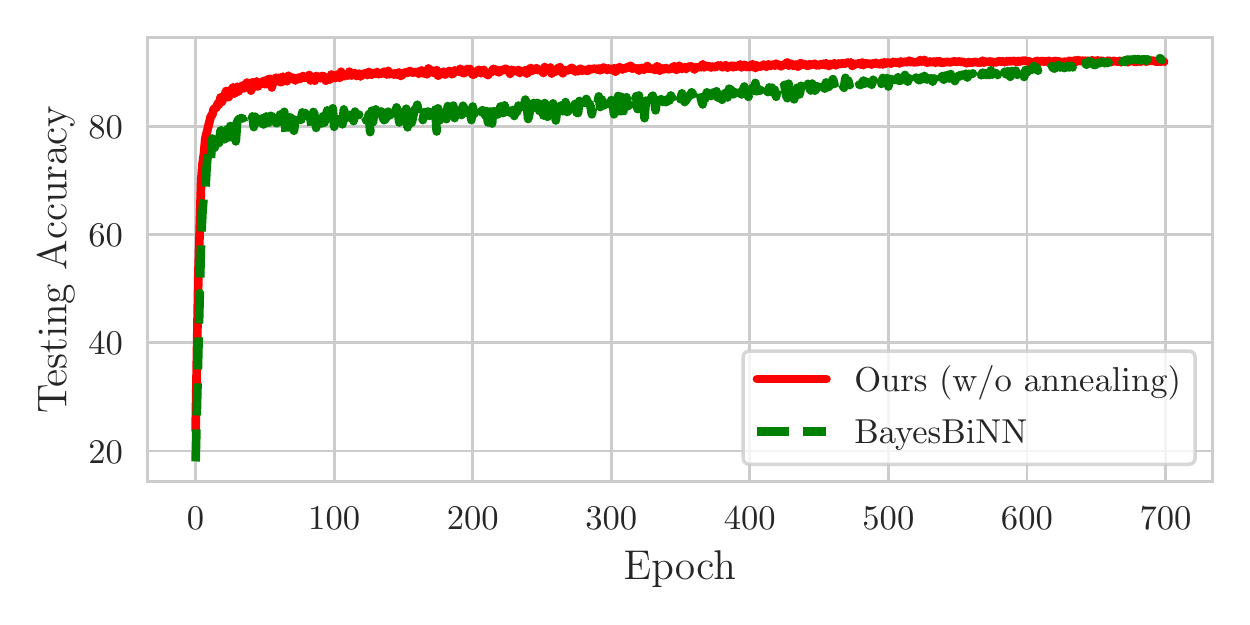}
    \includegraphics[width=0.49\textwidth]{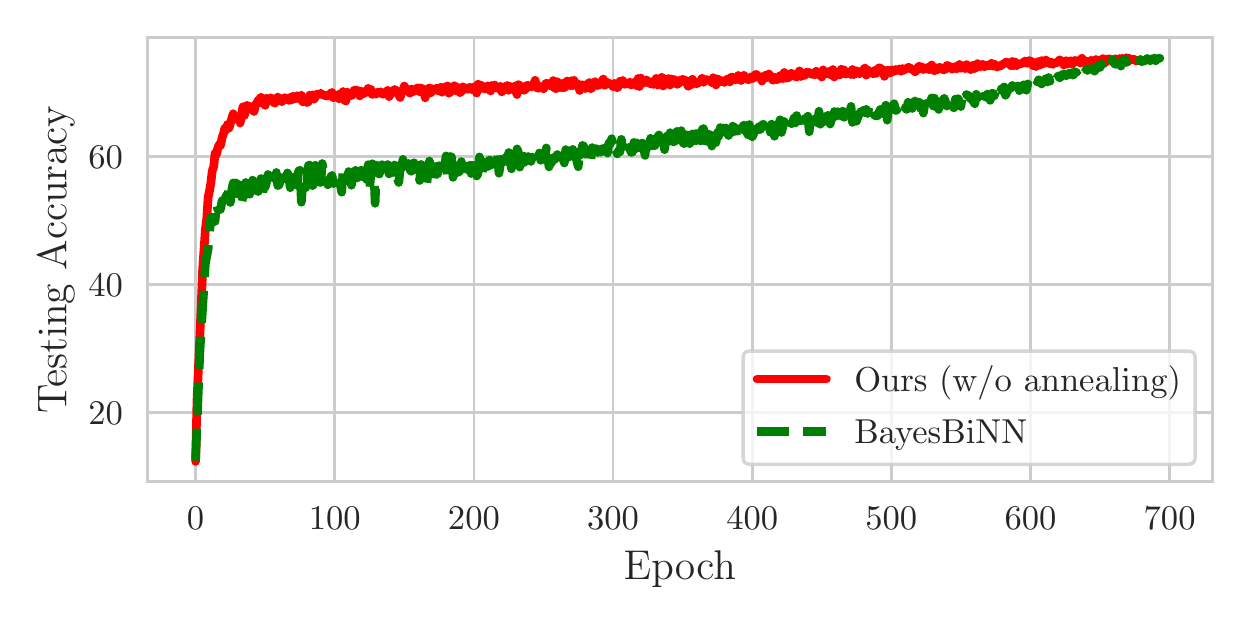}
    \includegraphics[width=0.49\textwidth]{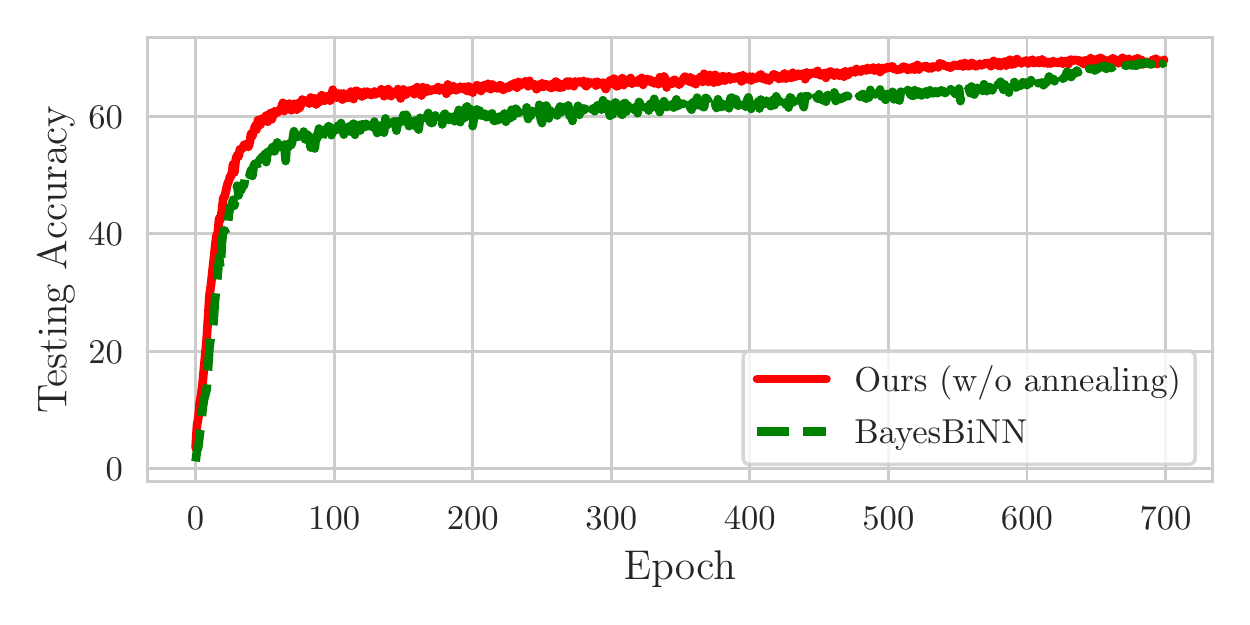}
    \caption{Testing accuracy achieved by the AdaSTE (no annealing) and BayesBiNN for 700 epochs. Top: CIFAR-10 with ResNet-18 (left) and VGG16 (right) }
    \label{fig:acc_700b}
\end{figure*}

\end{document}